\documentclass[letterpaper, 10 pt, conference]{ieeeconf}

\IEEEoverridecommandlockouts

\usepackage{multirow}
\usepackage{amsmath}
\usepackage{amsfonts}
\usepackage{amssymb}
\usepackage{amsxtra}
\usepackage{amsthm} 
\usepackage{leftidx}
\usepackage{url}
\usepackage{graphicx,color}
\usepackage{float}
\usepackage{stfloats}
\usepackage{dsfont}
\usepackage{mathrsfs}
\usepackage{array}
\usepackage{rotating}
\usepackage{calc}

\usepackage{enumerate}
\usepackage{tikz-cd}
\usepackage{dsfont}
\usepackage{stmaryrd}
\usepackage[linesnumbered,lined,commentsnumbered,ruled,norelsize]{algorithm2e}
\SetAlFnt{\footnotesize}
\SetAlCapFnt{\rmfamily\footnotesize}
\usepackage[capitalize,nameinlink]{cleveref}

\makeatletter
\newcommand{\removelatexerror}{\let\@latex@error\@gobble}
\makeatother

\usepackage{subcaption}

\newcounter{MYtempeqncnt}

\makeatletter
\def\@endtheorem{\endtrivlist}
\makeatother

\newtheorem{theorem}{Theorem}
\newtheorem{definition}{Definition}
\newtheorem{proposition}{Proposition}
\newtheorem{remark}{Remark}
\newtheorem{corollary}{Corollary}

\newtheorem{lemma}{Lemma}
\newtheorem{problem}{Problem}

\makeatletter
\AddToHook{cmd/appendices/before}{\def\cref@section@alias{appendix}}
\AddToHook{cmd/appendices/before}{\def\cref@subsection@alias{appendix}}
\makeatother

\date{\today}

\newcommand{\nn}{{\mathscr{N}\negthickspace\negthickspace\negthinspace\mathscr{N}}\negthinspace}
\newcommand{\ou}{%
  \mathrel{%
    \vcenter{\offinterlineskip
      \ialign{##\cr$<$\cr\noalign{\kern-1.5pt}$>$\cr}%
    }%
  }%
}

\renewcommand{\baselinestretch}{0.9}

\begin{document}

\title{\LARGE{\bf ShieldNN: A Provably Safe NN Filter for Unsafe NN Controllers}} %
\author{James Ferlez\textsuperscript{$*\dagger$}, Mahmoud Elnaggar\textsuperscript{$**\dagger$}, Yasser Shoukry\textsuperscript{$*$}, and Cody Fleming\textsuperscript{$***$}\\
\textsuperscript{$*$}Electrical Engineering and Computer Science, University of California, Irvine \\
\textsuperscript{$**$}Electrical and Computer Engineering, University of Virginia\\
\textsuperscript{$***$}Mechanical Engineering, Iowa State University
\thanks{\textsuperscript{$\dagger$} Equally contributing first authors
} %
\thanks{This work was partially sponsored by the NSF awards \#CNS-2002405, \#CNS-2013824 and \#CPS-1739333.
}
} %

\maketitle

\begin{abstract}
In this paper, we develop a novel closed-form Control Barrier Function (CBF) 
and associated controller shield for the Kinematic Bicycle Model (KBM) with 
respect to obstacle avoidance. The proposed CBF and shield --- designed by an 
algorithm we call ShieldNN --- provide two crucial advantages over existing 
methodologies. First, ShieldNN considers steering and velocity constraints 
directly with the non-affine KBM dynamics; this is in contrast to more general 
methods, which typically consider only affine dynamics and do not guarantee 
invariance properties under control constraints. Second, ShieldNN provides a 
closed-form set of safe controls for each state unlike more general methods, 
which typically rely on optimization algorithms to generate a single 
instantaneous control for each state. Together, these advantages make ShieldNN 
uniquely suited as an  efficient Multi-Obstacle Safe Actions (i.e. 
multiple-barrier-function  shielding) during training time of a Reinforcement 
Learning (RL) enabled Neural Network controller. We show via experiments that 
ShieldNN dramatically increases the completion rate of RL training episodes in 
the  presence of multiple obstacles, thus establishing the value of ShieldNN in 
training RL-based controllers.
\end{abstract}


\maketitle


\section{Introduction}
\label{sec:introduction}





Control Barrier Functions (CBF) \cite{ames2019control} and the associated idea 
of controller shielding \cite{cheng2019end,BrunkeSafeLearningRobotics2022,hsu2023safety,marvi2021safe,cohen2023safe} have 
become important tools in the design of safety-critical control systems, 
especially those that incorporate learning-enabled components such as Neural 
Networks (NNs). However most of these techniques both assume affine dynamics  
(including with assumptions on relative degree) and use optimization problems 
to create shielding control actions, especially under control-input 
constraints. As a result, these methods add the computational cost of an  
additional optimization problem at each time step in order to apply their 
controller shields; this computational cost is particularly relevant when 
training a Reinforcement Learning (RL) based controller, whose training 
episodes ideally run much faster than real time. Moreover, incorporating an 
optimization problem into the shielding operation makes it difficult to reason 
about the availability of a safe control action in the multi-CBF case 
\cite{BreedenCompositionsMultipleControl2023, AaliMultipleControlBarrier2022}.



In this paper, we propose the ShieldNN algorithm, which mitigates these 
downsides in the application-specific case of obstacle avoidance for the 
Kinematic Bicycle Model (KBM) (the KBM in turn being a good approximation for 
four-wheeled vehicles \cite{KongKinematicDynamicVehicle2015}). In particular, 
the ShieldNN algorithm can produce a closed-form, verified CBF with an 
associated optimization-free NN-based controller shield. Indeed, from the 
ShieldNN-designed CBF, it is effectively possible to obtain a \emph{set} of 
safe controls for any particular state of the KBM dynamics. Thus, a 
ShieldNN-designed shield has two significant advantages. First, it is an ideal 
controller shield for RL-based controller training, since it does not incur the 
cost of an optimization problem at each time step. Second, it can be easily 
extended from one obstacle (one CBF) to multiple obstacles case (multiple CBFs, 
one for  each obstacle), since it effectively outputs closed form \emph{sets} 
of safe controls; if the intersection of these sets is non-empty, then it 
provides a safe control for all of the obstacles simultaneously.

The main theoretical contribution of this paper is the ShieldNN algorithm, 
which is in turn based on a novel parameterized class of CBFs for the KBM 
dynamics and obstacle avoidance (see \cref{sec:zbf_for_kbm_dynamics}). Thus, 
ShieldNN takes as its input a particular instance of the KBM dynamics; a 
desired safety radius around an obstacle to be avoided; and a user-specified 
hyperparameter $\sigma \in [0,1]$, which adjusts the  ``aggressiveness'' of the 
resulting CBF and shield. ShieldNN then produces a CBF and controller shield 
using the following two components:

		\begin{minipage}[t]{\linewidth}
		\parshape 2 3pt \dimexpr\linewidth-15pt 10pt \dimexpr\linewidth-20pt\relax
		\textbf{ShieldNN Verifier }(\cref{sec:verifier})\textbf{:}  
		\vphantom{$A^{A^A}$} This component verifies that the user-supplied KBM 
		and safety radius are compatible with the chosen hyperparameter to 
		indeed create a CBF. In the process, this component also verifies 
		useful design properties of the set of safe controls admitted by the 
		barrier function. \emph{These properties  admit an effective 
		characterization of a guaranteed-safe set of  controls.}
		\end{minipage}\\

		\begin{minipage}[t]{\linewidth}
		\parshape 2 3pt \dimexpr\linewidth-15pt 10pt \dimexpr\linewidth-20pt\relax
		\textbf{ShieldNN Synthesizer }(\cref{sec:synthesizer})\textbf{:} The 
		result of the verification component is a barrier function where a 
		crucial boundary between safe and non-safe controls is known to be 
		convex in a particular state (of the  KBM). \emph{Thus, it is possible 
		to soundly (and tightly) approximate the set of safe controls as 
		function of KBM state, and to do so using a ReLU NN that approximates 
		this boundary.} A further clipping operation yields a controller shield 
		that provably restricts the control inputs to the set of safe controls 
		specified by the CBF.
		\end{minipage}\\

We further extend ShieldNN from the case of creating safe control actions to 
avoid a single obstacle to creating safe control actions to avoid multiple  
obstacles simultaneously. This extension appears in \cref{sec:mosa}, and it 
follows from running multiple instances of the ShieldNN controller shield 
simultaneously: i.e. one instance per obstacle using the relative state for 
each obstacle. A safe actions can then be obtained by using an action that is 
safe for all controller shield instances (and treating the  absence of such 
actions as a run-time diagnostic of un-safety). We validate  the performance of 
this controller shield via experiments detailed in \cref{sec:experiments}.

\noindent \textbf{Related Work.} %
For an extensive and recent survey of the field of ``safety shields (or 
filters),'' we refer the reader 
to~\cite{GuReviewSafeReinforcement2024,BrunkeSafeLearningRobotics2022}. 
Compared to the techniques reported 
in~\cite{GuReviewSafeReinforcement2024,BrunkeSafeLearningRobotics2022}, our 
proposed algorithms provide several critical advantages. In particular, and 
unlike other techniques, ShieldNN aims to find closed-form characterization of 
the set of safe control actions. This set of safe control actions can then be  
directly mapped into neural network layers, hence bypassing the need to solve 
an optimization problem, which affects the real-time execution of the control 
pipeline. Moreover, the obtained characterization of the safe set of control 
inputs leads to a natural extension of handling multiple obstacles at once,  
which is an issue for several of the techniques reported 
in~\cite{GuReviewSafeReinforcement2024,BrunkeSafeLearningRobotics2022}.

\section{Preliminaries} 
\label{sec:prelims}

\subsection{Notation} 
\label{sub:notation}
Let $\mathbb{R}$ denote the real numbers, and let $\lVert \cdot \rVert$ denote 
the Euclidean norm on $\mathbb{R}^n$. A ($K$-layer) ReLU network, is a function 
$\nn = (L_{\theta^{(K)}} \circ L_{\theta^{(K-1)}} \circ \dots \circ 
L_{\theta^{(1)}})(x)$ where $L_{\theta} : \mathbb{R}^{\mathfrak{i}} \rightarrow 
\mathbb{R}^{\mathfrak{o}}$ with $z \mapsto \max\{ W z + b, 0 \}$. and the 
$\max$ is taken element-wise and $\theta^{(k)} = (W_k, b_k)$. 
$L_{\theta^{(K)}}$ lacks the $\max$ operation by convention.


\subsection{Control Barrier Functions} 
\label{sub:control_barrier_functions}
We introduce the following corollary, \cref{cor:feedback_control_action_set}, 
which serves as the foundation of controller shielding 
\cite{AlshiekhSafeReinforcementLearning2017}.
\begin{corollary}
\label{cor:feedback_control_action_set}
	Let $h: \mathbb{R}^n \rightarrow \mathbb{R}$ with $\mathcal{C}_h \triangleq 
	\{x \in \mathbb{R}^n | h(x) \geq 0\}$, and let $\mathcal{D} \subseteq 
	\mathbb{R}^n$ s.t. $\mathcal{C}_h \subseteq \mathcal{D}$. Further let 
	$\dot{x} = f(x,u)$ be a control system where  $f : \mathbb{R}^n \times 
	\Omega_\text{admis.} \rightarrow \mathbb{R}^n$ is Lipschitz continuous. 
	Finally, let $\alpha$ be a class $\mathcal{K}$ function. If the set
	\begin{equation}
			R_{h,\alpha}(x) \triangleq \{ u \negthinspace \in \Omega_\text{admis.} | \nabla_x^\text{T} h(x) \cdot \negthinspace f(x, u) + \alpha(h(x)) \negthinspace \geq 0 \}
	\end{equation}
	is non-empty for each $x \in \mathcal{D}$, and a  Lipschitz-continuous 
	feedback controller $\mu : x \mapsto u$ satisfies
	\begin{equation}
		\mu(x) \in R_{h,\alpha}(x) \quad \forall x \in \mathcal{D}
	\end{equation}
	then $\mathcal{C}_h$ is forward invariant for dynamics 
	$f(\cdot,\negthinspace \mu(\cdot))$.
\end{corollary}
\begin{proof}
	This follows directly from an application of zeroing barrier functions 
	\cite[Theorem 3]{XuRobustnessControlBarrier2015}.
\end{proof}




\section{Problem Formulation} 
\label{sec:problem}

%

We consider the kinematic bicycle model (KBM) as our dynamical system model, 
%
however, we consider the KBM described in terms of vehicle-to-obstacle relative 
position variables (e.g. measurable via LiDAR). That is, the distance to the 
obstacle, $\lVert \vec{r} \rVert$, and the angle of the vehicle 
with respect to the obstacle, $\xi$, are states in the dynamical system:
\begin{align}
	\left(
		\begin{matrix}
			\dot{r} \\
			\dot{\xi} \\
			\dot{v}
		\end{matrix}
	\right)
	\negthickspace
	&=
	f_\text{KBM}\left(\left(\begin{smallmatrix}r \\ \xi\\ v\end{smallmatrix}\right) \negthinspace, \negthinspace \left(\begin{smallmatrix}a \\ \beta\end{smallmatrix}\right)\right)
	\triangleq \negthickspace
	\left(
		\begin{smallmatrix}
			v \cos( \xi - \beta ) \\
			-\frac{1}{r} v \sin(\xi - \beta) - \frac{v}{\ell_r} \sin(\beta) \\
			a 
		\end{smallmatrix}
	\right) \notag\\
	\beta &\triangleq \tan^{-1}(\tfrac{\ell_r}{\ell_f + \ell_r} \tan(\delta_f)); \;\;
	\chi \triangleq \left(\begin{smallmatrix}r \\ \xi\\ v\end{smallmatrix}\right); \;\;
	\omega \triangleq \left(\begin{smallmatrix}a \\ \beta\end{smallmatrix}\right)
	\label{eq:kbm_dynamics}
\end{align}
\begin{figure}[t]
    \vspace{1.8mm}
	\centering 
	\includegraphics[width=0.32\textwidth,trim={0cm 0.8cm 0cm 0cm},clip]{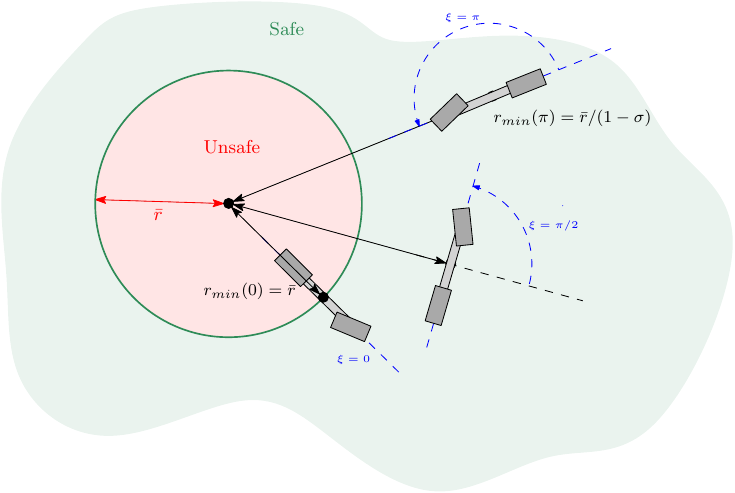} 
	\caption{Obstacle specification and minimum barrier distance as a function 
	of relative vehicle orientation, $\xi$. 
	} 
	\label{fig:obstacle_diagram}
\end{figure}  %
\noindent where $r(t) \triangleq \lVert \vec{r}(t) \rVert$; $a$ is the linear 
acceleration input; $\delta_f$ is the front-wheel steering angle 
input\footnote{That is the steering angle can be set instantaneously, and the 
dynamics of the steering rack can be ignored.}; and $\xi$ is the angle of the 
vehicle to the obstacle. For clarity, we note a few special cases: when $\xi = 
\pm \pi/2$, the vehicle is oriented tangentially to the obstacle, and when $\xi 
= \pi \text{ or } 0$, the vehicle is pointing directly at or away from the 
obstacle, respectively (see \cref{fig:obstacle_diagram}). $\beta$ is an 
intermediate quantity, an \emph{invertible function} of $\delta_f$.

We make the further assumption that the KBM has a control constraint on 
$\delta_f$ such that $\delta_f \in [-{\delta_f}_\text{max}, 
{\delta_f}_\text{max}]$. To simplify further notation, we will consider $\beta$ 
directly as a control variable; this is without loss of generality, since there 
is a bijection between $\beta$ and the actual steering control angle, 
$\delta_f$. Thus, $\beta$ is also constrained: $\beta \negthinspace \in 
\negthinspace [-\beta_\text{max}, \beta_\text{max}]$, so the control vector is 
constrained $\omega  \negthinspace \in \negthinspace \Omega_\text{admis.} 
\negthickspace \triangleq \negthickspace \mathbb{R} \times [-\beta_\text{max}, 
\beta_\text{max}]$. 





\begin{problem}
\label{prob:main_problem}
Consider a KBM vehicle with maximum steering angle ${\delta_f}_\text{max}$, 
length parameters $l_f = l_r$ and maximum velocity $v_\text{max}$\footnote{In 
our KBM model, this technically requires a feedback controller on $a$, but this 
won't affect our results. }. Consider also a disk-shaped region of radius 
$\bar{r}$ centered at the origin, $U = \{ x \in \mathbb{R}^2 : \lVert x \rVert 
\leq \bar{r} \}$.
	Find a set of safe initial conditions, $S_0$, and a ReLU NN:
	\begin{equation}
		\nn : (\chi, \omega) \mapsto \hat{\omega}
	\end{equation}
	such that for any globally Lipschitz continuous controller $\mu : \chi 
	\mapsto \omega \in \Omega_\text{admis.}$, the state feedback controller:
	\begin{equation}
	\label{eq:composed_controller}
		\nn(\chi, \mu(\chi)) : \chi \mapsto \hat{\omega}
	\end{equation}
	is guaranteed to prevent the vehicle from entering the unsafe region $U$ if 
	it was started from a state in $S_0$.
\end{problem}



\noindent We also consider the multi-obstacle extension of 
\cref{prob:main_problem}.

\begin{problem}
\label{prob:mosa}
Consider a KBM vehicle as in \cref{prob:main_problem}, but instead consider  
$P$ obstacles located at points $o_1, \dots, o_P \in \mathbb{R}^2$ and 
associated unsafe disks $U_p = \{x \in \mathbb{R}^2 : \lVert x - o_p \rVert 
\leq \bar{r}\}$, $p = 1, \dots, P$. Find a ReLU NN $\nn : (\chi, \omega) 
\mapsto \hat{\omega}$ and a safety monitoring function
\begin{equation}
	\label{eq:mosa_condition}
	\mathcal{S} : (\chi, \omega) \mapsto s \in \{0,1\}
\end{equation}
such that if $S(\chi, \omega) \negthinspace = \negthinspace 1$ along a 
trajectory of the KBM dynamics under state feedback controller $\nn(\chi, 
\mu(\chi)) : \chi \mapsto \hat{\omega}$, then the vehicle will not enter the 
unsafe set $\cup_{p =  1}^{P} U_p$.
\end{problem}

Our approach to solving \cref{prob:main_problem} is based on  
\cref{cor:feedback_control_action_set}, and analogous to other approaches to 
shielding \cite{AlshiekhSafeReinforcementLearning2017}. Thus, ShieldNN solves 
\cref{prob:main_problem} according to the following steps, using notation from 
\cref{cor:feedback_control_action_set} but dynamics \eqref{eq:kbm_dynamics}:


		\begin{minipage}[t]{\linewidth}
		\parshape 2 3pt \dimexpr\linewidth-15pt 10pt \dimexpr\linewidth-20pt\relax
		\textbf{(1) Design a Candidate Barrier Function }(\cref{sec:zbf_for_kbm_dynamics})\textbf{.} For 
		a function, $h$, to be a barrier function for a specific safety 
		property, its zero super-level set, $\mathcal{C}_h$, must be contained 
		in the set of safe states.
		\end{minipage}\\

		\begin{minipage}[t]{\linewidth}
		\parshape 2 3pt \dimexpr\linewidth-15pt 10pt \dimexpr\linewidth-20pt\relax
		\textbf{(2) Verify the Existence of Safe Controls }(\cref{sec:verifier})\textbf{.} 
		\emph{(ShieldNN Verifier)} Show that the set $R_{h,\alpha}(x)$ is 
		non-empty for each state $x \in \mathcal{C}_h$. This establishes that a 
		safe feedback controller may exist.
		\end{minipage}\\

		\begin{minipage}[t]{\linewidth}
		\parshape 2 3pt \dimexpr\linewidth-15pt 10pt \dimexpr\linewidth-20pt\relax
		\textbf{(3) Design a Safety Filter }(\cref{sec:synthesizer})\textbf{.} 
		\emph{(ShieldNN Synthesizer)} If possible, design $\nn_{\negthinspace 
		0}$ such that $\nn_{\negthinspace 0} : x \in \mathcal{C}_h \mapsto 
		\hat{u} \in R_{h,\alpha}(x)$; then obtain a safety filter as:
		\begin{equation}
			\nn(x,u) :=
			\begin{cases}
				u & \text{if } u \in R_{h,\alpha}(x) \\
				\nn_{\negthinspace 0}(x) & \text{if } u \not\in R_{h,\alpha}(x).
			\end{cases}
		\end{equation}
		\end{minipage}\\

Our approach to solving \cref{prob:mosa} (see \cref{sec:mosa}) is based on the 
availability of an estimate for the region $R_{h,\alpha}(x)$, as obtained from 
the ShieldNN CBF and controller shield (see above). In particular, we consider 
a ShieldNN-designed controller shield for each of the $P$ obstacles 
individually (after shifting the KBM dynamics so that  $o_p$ is the origin), 
which provides a per-obstacle set of set of safe controls,  
$R^{o_p}_{h,\alpha}(x)$, $p = 1, \dots, P$. Testing whether the intersection 
$\cap_{p=1}^{P} R^{o_p}_{h,\alpha}(x)$ is nonempty yields the desired safety 
monitor $S$, and choosing a control from such a nonempty intersection yields a 
safety filter.


 %




\section{Barrier Function(s) for the KBM Dynamics: the Basis of ShieldNN} 
\label{sec:zbf_for_kbm_dynamics}

We propose the following class of candidate barrier functions to certify 
control actions so that the vehicle 
doesn't get within $\bar{r}$ units 
of the origin (\cref{prob:main_problem}):
\begin{equation}
	\label{eq:bicycle_barrier}
	h_{\bar{r},\sigma}(\chi) = h_{\bar{r},\sigma}(\xi, r, v) = \frac{\sigma \cos(\xi/2) + 1-\sigma}{\bar{r}} - \frac{1}{r}
\end{equation}
where $\sigma \in (0,1)$ is an additional parameter whose function we shall 
describe subsequently. First note that the equation $h_{\bar{r},\sigma}(\chi) = 
0$ has a unique solution, $r_\text{min}(\xi)$ for each of $\xi$: 
\begin{equation}
\label{eq:r_min}
	r_\text{min}(\xi) = \bar{r}/( \sigma \cos(\xi/2) + 1-\sigma ),
\end{equation}
so the smallest value of $r_\text{min}$ is $r_\text{min}(0) = \bar{r}$. Thus, 
the function $h_{\bar{r},\sigma}$ satisfies the requirements of \textbf{(1)} in 
the ShieldNN framework: i.e. $\mathcal{C}_{h_{\bar{r},\sigma}}$, the zero 
super-level set of $h_{\bar{r},\sigma}$, is entirely contained in the set of 
safe states as proscribed by \cref{prob:main_problem}, independent of the 
choice of $\sigma$. See \cref{fig:obstacle_diagram}, which also depicts another 
crucial value, $r_\text{min}(\pm \pi) = \bar{r}/(1-\sigma)$. 

\begin{remark}
	Note that $h_{\bar{r},\sigma}$ is independent of the velocity state, $v$. 
	This will ultimately force ShieldNN filters to intervene only by altering 
	the steering input.  
\end{remark}
\begin{figure*}[!b]
\hrulefill
\normalsize
\setcounter{MYtempeqncnt}{\value{equation}}
\setcounter{equation}{12}
\begin{multline}
	\mathcal{L}_{\bar{r},\sigma,\ell_r}(\xi, \beta, v) \triangleq \Big[
	\nabla_\chi^\text{T} h_{\bar{r},\sigma}(\chi) \cdot f_\text{KBM}(\chi, \omega) 
		+ \alpha(h_{\bar{r},\sigma}(\chi)) \Big]_{\chi = (r_\text{min}(\xi),\xi,v)} \\
	= v \Big(
			\tfrac{\sigma}{2 \cdot \bar{r} \cdot r_\text{min}(\xi)} \sin(\tfrac{\xi}{2}) \sin(\xi \negthinspace - \negthinspace \beta)
			+ 
			\tfrac{\sigma}{2 \cdot \bar{r} \cdot \ell_r} \sin(\tfrac{\xi}{2}) \sin(\beta)
			+ \tfrac{1}{r_\text{min}(\xi)^2}\cos(\xi \negthinspace - \negthinspace \beta)
		\Big)
	\label{eq:zbf_full}
\end{multline}
\setcounter{equation}{9}
\vspace*{4pt}
\end{figure*}

A barrier function also requires a class $\mathcal{K}$ function, $\alpha$. For 
ShieldNN, we choose a linear function
\begin{equation}
\label{eq:extended_class_k}
	\alpha_{v_\text{max}}(x) = K \cdot v_\text{max} \cdot x
\end{equation}
where $v_\text{max}$ is the assumed maximum velocity (\cref{prob:main_problem}), and constant $K$ is a selected according to the 
following.
\begin{theorem}
\label{thm:safe_control_region_grows}
	Consider any fixed $\bar{r}$, $\ell_r$ and $\sigma$. Assume that $0 \leq v 
	\leq v_\text{max}$ (as specified by \cref{prob:main_problem}). If $K$ is  
	such that:  
	\begin{equation}
			K \geq K_{\bar{r},\sigma} \triangleq \max(\{1,1/\bar{r}\}) \cdot \left( \tfrac{\sigma}{2 \cdot \bar{r} } 
		+ 2 \right)
	\end{equation}
	then the Lie derivative $\nabla_\chi^\text{T} h_{\bar{r},\sigma}(x) \cdot 
	f_\text{KBM}(\chi, \omega) + \alpha(h_{\bar{r},\sigma}(\chi))$ 
	is a monotonically increasing function in $r$ for all $r \geq \bar{r}$ for 
	each fixed choice of $v \in (0,v_\text{max}]$ and the remaining state and 
	control variables. 

	In particular, for all $\chi \in \mathcal{C}_{h_{\bar{r},\sigma}}$ such 
	that $v \in (0, v_\text{max}]$ it is the case that: 
	\begin{equation}\label{eq:simple_safe_controls}
		R_{h_{\bar{r},\sigma}}((r_\text{min}(\xi),\xi,v)) \subseteq R_{h_{\bar{r},\sigma}}(\chi).
	\end{equation}
	where $R_{h_{\bar{r},\sigma}}$ has an analogous definition to that in  
	\cref{cor:feedback_control_action_set} (but with the $\alpha$ subscript 
	suppressed for brevity).
\end{theorem}

\begin{proof}
	See \cref{sub:proof_of_thm:safe_control_region_grows} of 
	\cite{ferlez2024shieldnnprovablysafenn}.
\end{proof}

As a consequence of \cref{thm:safe_control_region_grows}, define the  
following.

\begin{definition}
	\label{def:deriv_condition_rmin}
	Define $\mathcal{L}_{\bar{r},\sigma,\ell_r}(\xi, \beta, v)$ according to 
	\cref{eq:zbf_full}. 
\end{definition}
\setcounter{equation}{13}

In addition to concretely defining our class of candidate barrier functions,  
\cref{thm:safe_control_region_grows} is the essential facilitator of the 
ShieldNN algorithm. In particular, note that the right-hand side of  
\eqref{eq:zbf_full} can be simplified as indicated,
since $h_{\bar{r},\sigma}((r_\text{min}(\xi),\xi,v)) = 0$ and 
$\alpha_{v_\text{max}}(0) = 0$. Hence, the set 
$R_{h_{\bar{r},\sigma}}\negthinspace ((r_\text{min}(\xi),\xi,v))$ is 
independent of $v$, so \eqref{eq:simple_safe_controls} gives a sufficient 
condition for safe controls \textbf{(2)} in terms of a single state variable, 
$\xi$, and a single control variable $\beta$. This simplifies not only the 
ShieldNN verifier but also the ShieldNN synthesizer, as we demonstrate in the 
next section.



\begin{figure}[tp]
    \centering
    \begin{subfigure}[p]{0.46\linewidth}
		\vspace{1.95mm}
		\includegraphics[width=0.89\linewidth,trim={1.1cm 0.1cm 0cm 0.7cm},clip]{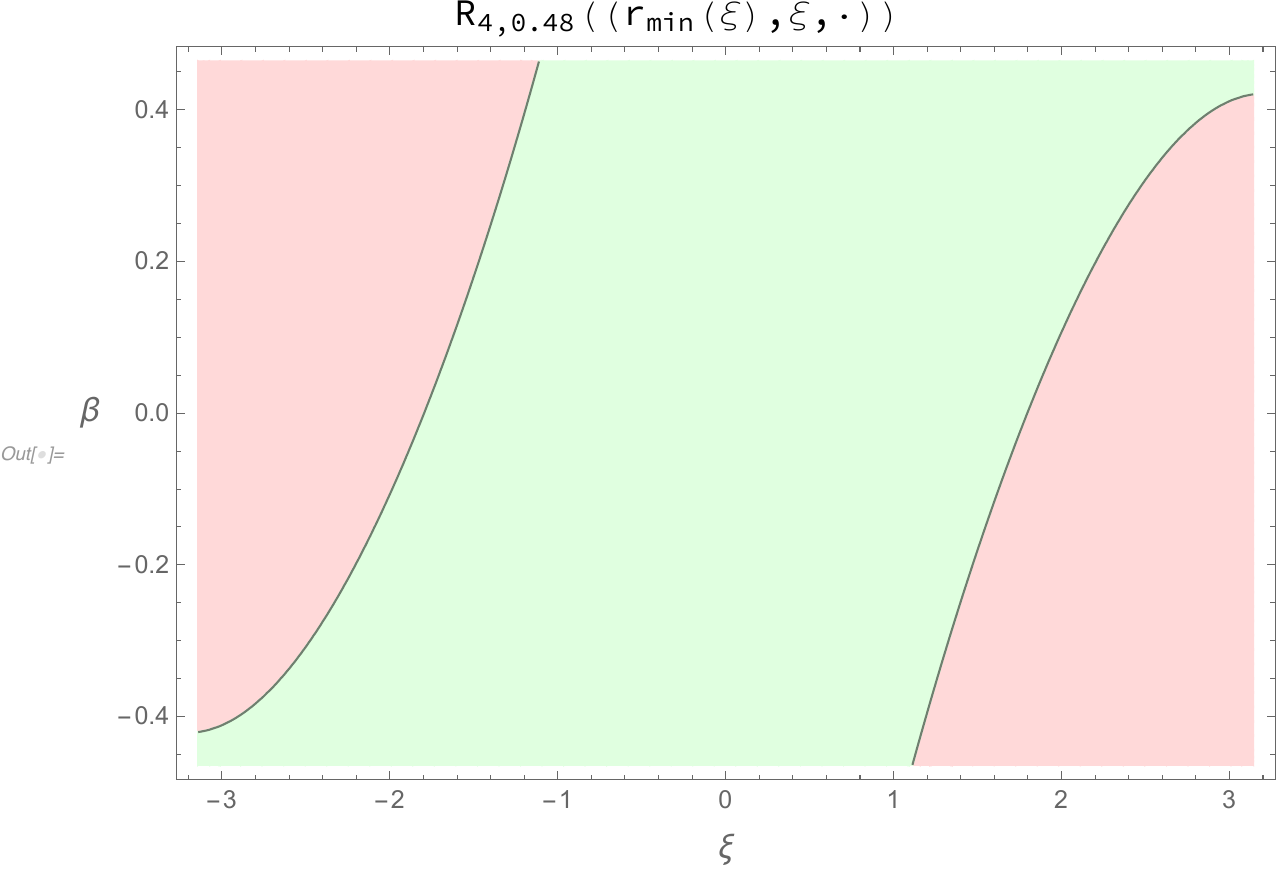}%
	\vspace{-1.5mm} %
	\caption{(Un)Safe steering controls. $\cup_\xi R_{\scriptscriptstyle 4,.48}((r_\text{min}\negthinspace(\xi),\negthinspace\xi,\negthinspace\cdot))$
	in light green; $\mathfrak{l}$ and $\mathfrak{u}$ in dark green.}
	\label{fig:practical_safe_controls}
	\end{subfigure}
	%
	\hspace{2pt}
	\begin{subfigure}[p]{0.46\linewidth}
		\includegraphics[width=1.02\linewidth,trim={1.1cm 0.15cm 1.5cm 1.2cm},clip]{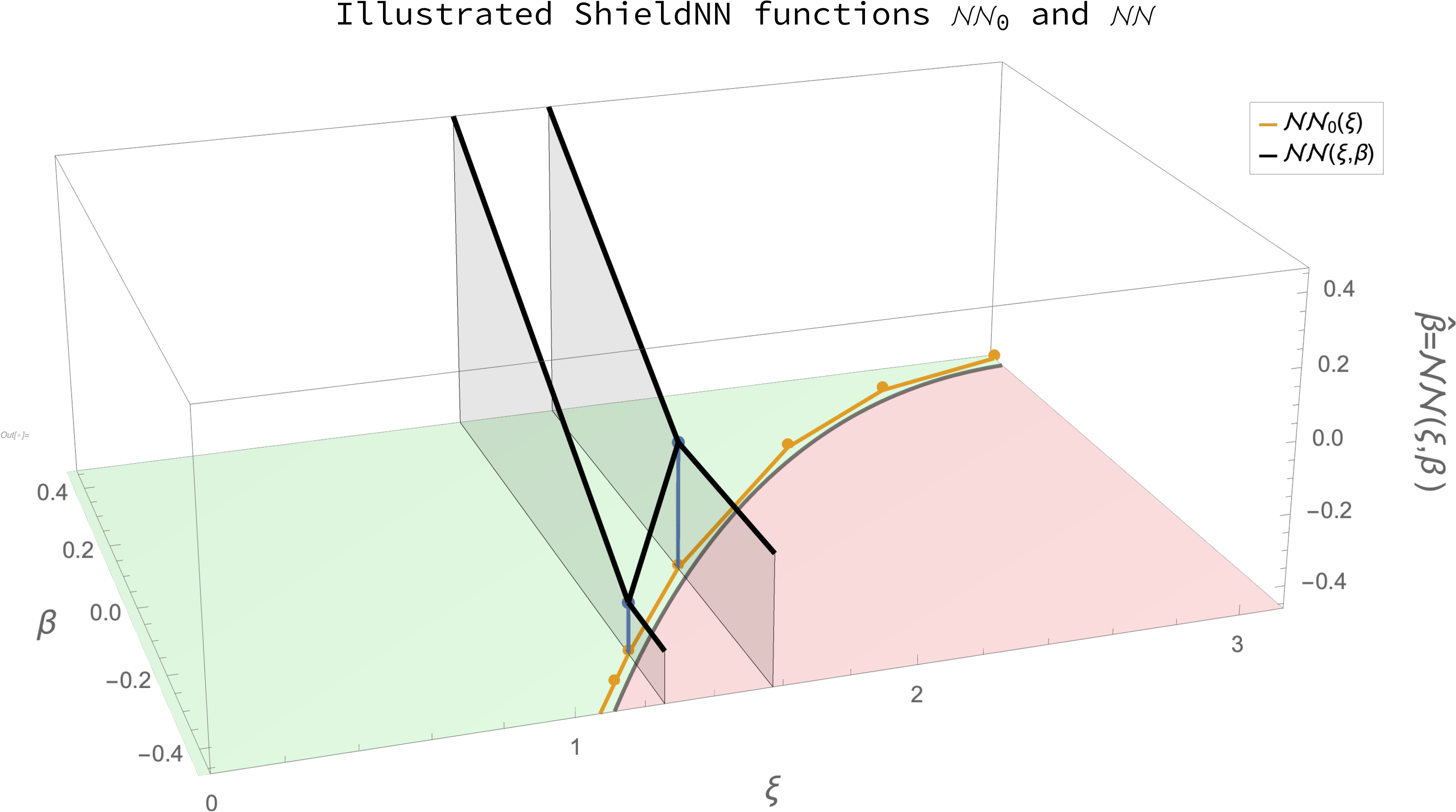}
		\caption{ %
		$\nn_{\negthinspace 0}$ (orange) and two constant-$\xi$ slices of the final ShieldNN filter, $\nn$ (black).
		} %
		\label{fig:shieldnn_design}
	\end{subfigure}
	\caption{Illustrated ShieldNN products for $\ell_f = \ell_r = 2 \;\text{m}$, $\bar{r} = 4 \;\text{m}$, $\beta_\text{max} = 0.4636$, $\sigma = 0.48$.}
\end{figure}  %

\section{ShieldNN Verifier} %
\label{sec:verifier}
The overall ShieldNN algorithm has three inputs: the specs for a KBM vehicle 
($\ell_f=\ell_r$, $\delta_{f,\text{max}}$ and $v_\text{max}$); the desired 
safety radius ($\bar{r}$); and the barrier parameter $\sigma$. Thus, the 
objective of the ShieldNN verifier is to soundly verify that these parameter 
values make equation \eqref{eq:bicycle_barrier} a true barrier function for 
\cref{prob:main_problem}.  From \cref{thm:safe_control_region_grows}, it 
suffices to show that 
$R_{h_{\bar{r},\sigma}}\negthinspace((r_\text{min}(\xi),\xi,\cdot)) \neq  
\emptyset$ for each $\xi \in [-\pi,\pi]$; the barrier property then follows  by 
\cref{cor:feedback_control_action_set}.

However, the ShieldNN verifier soundly verifies that each  
$R_{h_{\bar{r},\sigma}}\negthinspace((r_\text{min}(\xi),\xi,\cdot))$ is 
nonempty by soundly verifying that these sets have certain structural  
properties. In particular, ShieldNN verifies that for each orientation angle, 
$\xi$, the set 
$R_{h_{\bar{r},\sigma}}\negthinspace((r_\text{min}(\xi),\xi,\cdot))$ is an 
\emph{interval} whose bounds are a continuous function of $\xi$ and also 
clipped at the max/min steering inputs: i.e.
\begin{equation}
	R_{h_{\bar{r},\sigma}}\negthinspace ((r_\text{min}(\xi),\xi,\negthinspace\cdot))
	\negthinspace = \negthinspace
	\big[ \negthinspace
		\max\{ -\beta_\text{max}, \negthinspace \mathfrak{l}(\xi) \},
		\min\{  \beta_\text{max}, \negthinspace \mathfrak{u}(\xi) \} 
	\big]
	\label{eq:safe_control_intervals}
\end{equation}
where $\mathfrak{l}$ and $\mathfrak{u}$ are continuous functions of $\xi$.  
Moreover, ShieldNN verifies that the function $\mathfrak{l}$ (resp.  
$\mathfrak{u}$) is \emph{concave} (resp. convex); indeed, the symmetry of the 
problem dictates that $\mathfrak{u}(\xi) = -\mathfrak{l}(-\xi)$. From these  
structural properties, it becomes straightforward to establish 
$R_{h_{\bar{r},\sigma}}\negthinspace((r_\text{min}(\xi),\xi,\cdot)) \neq  
\emptyset$, and hence the desired conclusion as described above.

We reiterate that the ShieldNN verifier is only \emph{sound}, and so may  fail 
to verify the structural properties described above. However, this has not been 
observed in practice. See \cref{fig:practical_safe_controls} for an 
illustrative example with parameters $\ell_f = \ell_r = 2 \;\text{m}$, $\bar{r} 
= 4 \;\text{m}$, $\beta_\text{max} = 0.4636$ and $\sigma = 0.48$; $\cup_{\xi\in 
[-\pi,\pi]}R_{h_{\bar{r},\sigma}} \negthinspace 
((r_\text{min}(\xi),\xi,\cdot))$ is shown in light green, and $\mathfrak{l}$ 
and $\mathfrak{u}$ are shown in dark green.



\subsection{Algorithmic Verification of~\eqref{eq:safe_control_intervals}} 
Recall that the main function of the ShieldNN verifier is to soundly verify 
that equation~\eqref{eq:safe_control_intervals} holds
for a concave function $\mathfrak{l}$ and with $\mathfrak{u}(\xi) = 
-\mathfrak{l}( -\xi )$. The conclusion about $\mathfrak{u}$ follows directly 
from the symmetry of the problem, so we focus on the claims for $\mathfrak{l}$.

We now make the following observation.

\begin{proposition}
Suppose that \eqref{eq:safe_control_intervals} holds with $\mathfrak{u}(\xi) = 
-\mathfrak{l}(-\xi)$, and let $\mathcal{L}_{\bar{r},\sigma,\ell_r}$ be as in  
\cref{def:deriv_condition_rmin}. Then for any $\xi^\prime \in [-\pi,\pi]$ such that 
$\mathfrak{l}(\xi^\prime) \in (-\beta_\text{max},\beta_\text{max})$ it is the 
case that
\begin{equation}\label{eq:main_functional_motivator}
	\mathcal{L}_{\bar{r},\sigma,\ell_r}(\xi^\prime, \mathfrak{l}(\xi^\prime), \cdot) = 0.
\end{equation}
\end{proposition}
\begin{proof}
	This follows from the definition of $R_{h_{\bar{r},\sigma}}$, and the fact 
	that we are evaluating it \emph{on} the barrier, i.e. for $\chi^\prime = 
	(r_\text{min}(\xi^\prime),\xi^\prime,v)$. Thus, $h(\chi^\prime) = 0$, and   
	$\alpha_{v_\text{max}}(h(\chi^\prime)) = 0.$
\end{proof}

This suggests that  \eqref{eq:main_functional_motivator} can be used to 
establish the claim in \eqref{eq:safe_control_intervals}. Thus, let $a<b$ be 
real numbers, and define: 
\begin{equation}
	\mathsf{bd}_{[a,b]} \negthinspace \triangleq \negthinspace \{ (\xi^\prime\negthinspace,\beta^\prime) \negthinspace \in \negthinspace [a,\negthinspace b] \negthinspace \times \negthinspace [-\beta_\text{max}, \beta_\text{max}]
	|
	\mathcal{L}_{\bar{r},\sigma,\ell_r}(\xi^\prime, \beta^\prime, \cdot) \negthinspace = \negthinspace 0
	\}
\end{equation}
with the appropriate modifications for other interval types $(a,b)$, $(a,b]$ 
and $[a,b)$. We also define a related quantity:
\begin{equation}
	\text{dom}(\mathsf{bd}_{[a,b]}) = \{\xi \in [a,b] \; | \; \exists \beta . (\xi,\beta) \in \mathsf{bd_{[a,b]}}\}.
\end{equation}
We thus develop a sound algorithm to verify \eqref{eq:safe_control_intervals} 
and the concavity of $\mathfrak{l}$ by soundly verifying these properties in  
order:

	\begin{minipage}[t]{\linewidth}
	\parshape 2 3pt \dimexpr\linewidth-15pt 10pt \dimexpr\linewidth-20pt\relax
	\textbf{Property 1.} %
	Show that $\mathsf{bd}_{[-\pi,\pi]} \cap \big([-\pi,\pi] \times 
	\{-\beta_\text{max}\}\big) = \{ (\xi_0, \beta_\text{max}) \}$; that is 
	$\mathsf{bd}_{[-\pi,\pi]}$ intersects the lower control constraint a single 
	orientation angle, $\xi_0$. Likewise $\mathsf{bd}_{[-\pi,\pi]} \cap 
	\big([-\pi,\pi] \negthinspace \times \negthinspace {\beta_\text{max}}\big)  
	\negthinspace = \negthinspace \{ \negthinspace (-\xi_0, \beta_\text{max}) 
	\negthinspace \}$ by symmetry. 
	\end{minipage}\\

	\begin{minipage}[t]{\linewidth}
	\parshape 2 3pt \dimexpr\linewidth-15pt 10pt \dimexpr\linewidth-20pt\relax
	\textbf{Property 2.} %
	Verify that $\mathsf{bd}_{[\xi_0,\pi]}$ is the graph of a function 
	(likewise for $\mathsf{bd}_{[-\pi,-\xi_0]}$ by symmetry), and  
	$\mathsf{bd}_{(-\xi_0,\xi_0)} = \emptyset$. Thus, define $\mathfrak{l}$ as 
	$\text{graph}(\mathfrak{l}) \triangleq \mathsf{bd}_{[\xi_0,\pi]}$. 
	\end{minipage}\\

	\begin{minipage}[t]{\linewidth}
	\parshape 2 3pt \dimexpr\linewidth-15pt 10pt \dimexpr\linewidth-20pt\relax
	\textbf{Property 3.} %
	Verify that $\mathfrak{l}$ in \textbf{Property 2} is concave.
	\end{minipage} \\

The ShieldNN verifier expresses each of these properties as the sound 
verification that a particular function is greater than zero on a subset of its 
domain. The functions that are associated with these properties are either 
$\mathcal{L}_{\bar{r},\sigma,\ell_r}$ itself or else derived from it (i.e. by 
differentiating), so each is an analytic function where $\xi$ and $\beta$ 
appear only in trigonometric functions. These surrogate verification problems 
are amenable to over-approximation and the Mean-Value Theorem.

With this program in mind, \cref{ssub:property_1} - \cref{ssub:property_3} 
explain how to express \textbf{Property 1}-\textbf{3} as minimum-verification 
problems. \cref{ssub:certifymin} describes main algorithmic component of the 
ShieldNN verifier, \texttt{CertifyMin}.

\subsection{Verifying Property 1} 
\label{ssub:property_1}
To verify \textbf{Property 1}, we can start by using a numerical root finding 
algorithm to find a zero of $\mathcal{L}_{\bar{r},\sigma,\ell_r}(\xi, -\beta_\text{max}, \cdot)$, 
viewed as a function of $\xi$ (\cref{def:deriv_condition_rmin}). However, there 
is no guarantee that this root, call it $\hat{\xi}_0$ is the only root on the 
set $[-\pi,\pi] \times \{-\beta_\text{max}\}$. In this case, the property to be 
verified reduces to the following.

\begin{proposition}\label{prop:single_zero}
	Suppose that $\mathcal{L}_{\bar{r},\sigma,\ell_r}(\hat{\xi}_0, 
	-\beta_\text{max}, \cdot) = 0$  (\cref{def:deriv_condition_rmin}). Also, 
	suppose that there exists an $\epsilon > 0$ such that:

	\begin{enumerate}
		\item $\forall \xi \in [-\pi,\hat{\xi}_0-\epsilon] \; . \; 
			\mathcal{L}_{\bar{r},\sigma,\ell_r}(\hat{\xi}_0-\epsilon, -\beta_\text{max}, \cdot) > 0$;

		\item  
			$\mathcal{L}_{\bar{r},\sigma,\ell_r}\negthinspace(\hat{\xi}_0-\epsilon, 
			-\beta_\text{max}, \cdot) > 0$ and 
			$\mathcal{L}_{\bar{r},\sigma,\ell_r}\negthinspace(\pi, 
			-\beta_\text{max}, \cdot) \negthinspace < \negthinspace 0$;

		\item $\forall \xi \in [\hat{\xi}_0 - \epsilon, \pi] \;.\; 
			\tfrac{\partial^2}{\partial \xi^2} \mathcal{L}_{\bar{r},\sigma,\ell_r}(\xi, 
			-\beta_\text{max}, \cdot) > 0$.
	\end{enumerate}
	Then $\hat{\xi}_0$ is the only root of $\mathcal{L}(\xi, -\beta_\text{max}, 
	\cdot)$ on $[-\pi,\pi] \times \{-\beta_\text{max}\}$. That is 
	\textbf{Property 1} is verified.
\end{proposition}
\begin{proof}
	If \emph{(i)} is true, then no zeros of 
	$\mathcal{L}_{\bar{r},\sigma,\ell_r}$ lie in $[-\pi, \hat{\xi}_0-\epsilon]$.

	If \emph{(iii)} is true, then $\mathcal{L}_{\bar{r},\sigma,\ell_r}(\xi, 
	-\beta_\text{max}, \cdot)$ is a convex function of $\xi$ on the interval 
	$[\hat{\xi}_0, \pi]$. But if \emph{(ii)} is also true, then $\hat{\xi}_0$ 
	must be the only zero of $\mathcal{L}_{\bar{r},\sigma,\ell_r}$ on the same 
	interval. This follows by contradiction from the assertion of convexity. If 
	there were another zero on $(\hat{\xi}_0, \pi]$, then the line connecting 
	$(\hat{\xi}_0, \mathcal{L}_{\bar{r},\sigma,\ell_r}(\hat{\xi}_0, 
	-\beta_\text{max}, \cdot))$ and $(\pi, 
	\mathcal{L}_{\bar{r},\sigma,\ell_r}(\pi, -\beta_\text{max}, \cdot))$ would 
	lie below this point by assumption \emph{(ii)}, hence contradicting 
	convexity. A similar argument can be made if there were a zero on 
	$[\hat{\xi}_0-\epsilon, \hat{\xi}_0)$.
\end{proof}

Crucially, the conditions \emph{(i)}-\emph{(iii)} of \cref{prop:single_zero} 
are conditions that can be checked either by verifying that a function is 
greater than 0 on an interval (as for \emph{(ii)} and \emph{(iii)}), or else  
$\mathcal{L}_{\bar{r},\sigma,\ell_r}$ has a particular sign for particular 
inputs (as in \emph{(i)}). Thus, ShieldNN verifier can establish 
\textbf{Property 1} by means of the \texttt{CertifyMin} function that we 
propose later.

\subsection{Verifying Property 2} 
\label{ssub:property_2}
Our verification of \textbf{Property 2} depends on the conclusion of 
\textbf{Property 1}. In particular, let $\xi_0 = \hat{\xi}_0$ be the single 
root of $\mathcal{L}_{\bar{r},\sigma,\ell_r}$ on $\big([-\pi,\pi] \times 
\{-\beta_\text{max}\}\big)$ as verified above. As before, the proposition below 
gives sufficient conditions to assert \textbf{Property 2}, and where verifying 
those conditions requires at worst checking the sign of some 
$\mathcal{L}_{\bar{r},\sigma,\ell_r}$-derived function on an interval (or 
rectangle).

The main technique for proving that $\mathsf{bd}_{[\xi_0,\pi]}$ is the graph of 
a function is to note that constant-level curves of 
$\mathcal{L}_{\bar{r},\sigma,\ell_r}$ are solutions to the ODE defined by its 
gradient. In particular, then, $\mathsf{bd}_{[\xi_0, \pi]}$ contains such a 
solutions in the rectangle of interest, since it is a subset of the 
zero-constant level curve of $\mathcal{L}_{\bar{r},\sigma,\ell_r}$. Thus, we 
can verify the desired properties of $\mathsf{bd}_{[\xi_0, \pi]}$ by 
considering the aforementioned ODE, and demonstrating that it has only one 
solution in the rectangle of interest. 
\begin{proposition}
\label{prop:functional_verification}
	Let $\xi_0$ be as above. Now suppose the following two conditions are 
	satisfied:
	\begin{enumerate}
		\item $\tfrac{\partial}{\partial \xi} 
			\mathcal{L}_{\bar{r},\sigma,\ell_r}(\xi_0,-\beta_\text{max}, \cdot) > 0$;

		\item for all $(\xi,\beta) \in \big([\xi_0, \pi] \times 
			[-\beta_\text{max}, \beta_\text{max}]\big)$ it is true that 
			\begin{equation}
			\label{eq:first_verification_prop}
				\tfrac{\partial}{\partial \beta} \mathcal{L}_{\bar{r},\sigma,\ell_r}(\xi,\beta, \cdot) < 0; \text{ and}
			\end{equation}

		\item there exists $\epsilon > 0$ and $\hat{\beta}_0 \in 
			[-\beta_\text{max}, \beta_\text{max}]$ such that

			\begin{enumerate}
				\item $\forall \xi \in 
					[-\beta_\text{max},\hat{\beta}_0-\epsilon] \; . \; 
					\mathcal{L}_{\bar{r},\sigma,\ell_r}(\pi,\hat{\beta}_0-\epsilon, \cdot) < 0$;


				\item $\forall \beta \in [\hat{\beta}_0 - \epsilon, 
					\beta_\text{max}] \;.\; \tfrac{\partial}{\partial \beta} 
					\mathcal{L}_{\bar{r},\sigma,\ell_r}(\pi, \beta, \cdot) > 0$.
			\end{enumerate}
	\end{enumerate}
	Then $\mathsf{bd}_{[\xi_0, \pi]}$ is the graph of a function, and we can 
	define the function $\mathfrak{l}$ on $[\xi_0, \pi]$ by 
	$\text{graph}(\mathfrak{l}) \triangleq \mathsf{bd}_{[\xi_0,\pi]}$.
\end{proposition}
\begin{proof}
	Consider the two-state ODE defined by:
	\begin{equation}
		\dot{\xi} = -\tfrac{\partial}{\partial \beta} \mathcal{L}_{\bar{r},\sigma,\ell_r}(\xi, \beta, \cdot) ;\quad
		\dot{\beta} = \tfrac{\partial}{\partial \xi} \mathcal{L}_{\bar{r},\sigma,\ell_r}(\xi, \beta, \cdot). \label{eq:constant_curve_ode}
	\end{equation}
	The solutions to \eqref{eq:constant_curve_ode} are guaranteed to exist and 
	be unique on $[\xi_0, \pi] \times [-\beta_\text{max}, \beta_\text{max}]$, 
	since the vector field is locally Lipschitz on that rectangle (it is 
	differentiable). Thus, any solution of \eqref{eq:constant_curve_ode} is 
	guaranteed to follow a constant-level curve of 
	$\mathcal{L}_{\bar{r},\sigma,\ell_r}$ within this rectangle; the particular 
	constant-level curve matches the value of 
	$\mathcal{L}_{\bar{r},\sigma,\ell_r}$ for its initial condition.

	As a first step, we establish two facts about the solution of 
	\eqref{eq:constant_curve_ode} with initial condition $(\xi_0, 
	-\beta_\text{max})$:
	\begin{enumerate}
		\item the $\beta$ component of this solution strictly increases; and 

		\item the solution exits $[\xi_0, \pi] \times [-\beta_\text{max}, 
			\beta_\text{max}]$ only through its $\xi = \pi$ edge.
	\end{enumerate}

	First note that some initial portion of this solution must be contained in 
	$[\xi_0, \pi] \times [-\beta_\text{max}, \beta_\text{max}]$ by assumption 
	\emph{(i)} and assumption \emph{(ii)} applied to $(\xi_0, 
	-\beta_\text{max})$. Statement 1 is thus established directly by assumption 
	\emph{(ii)}. Now we establish 2). Note that the solution cannot exit via 
	the $\xi=\xi_0$ edge because its $\beta$ component is strictly increasing. 
	And it can't exit the $\beta = \mp \beta_\text{max}$ edges either, because 
	we have verified that $\mathcal{L}_{\bar{r},\sigma,\ell_r}$ has only one 
	root on each edge, $(\xi, -\beta_\text{max})$ and $(-\xi_0, 
	\beta_\text{max})$, respectively. This leaves only the $\xi=\pi$ edge. The 
	solution must leave this rectangle eventually, by the exclusion of the 
	other edges and the fact that its $\beta$ component is strictly increasing. 
	Thus, it exits via the $\xi = \pi$ edge.

	The final conclusion about the functionality follows if 
	$\mathsf{bd}_{[\xi_0, \pi]}$ if $\mathsf{bd}_{[\xi_0, \pi]}$ corresponds 
	\emph{exactly} to the single, unique solution described above. To verify 
	this, we need to verify that there is a single root of 
	$\mathcal{L}_{\bar{r},\sigma,\ell_r}$ along the $\xi=\pi$ edge, much as we 
	did to verify \textbf{Property 1}; this is made possible by 
	assumption\emph{(iii)}, items \emph{(I)}-\emph{(II)}.
\end{proof}

As in the case of \cref{prop:single_zero}, the conditions of 
\cref{prop:functional_verification} are conditions that can be checked using 
the \texttt{CertifyMin} function that we will propose subsequently.


\subsection{Verifying Property 3} 
\label{ssub:property_3}

We verify \textbf{Property 3} starting from the assumption that verifications 
of \textbf{Property 1} and \textbf{Property 2} were successful. In particular, 
we assume a function $\mathfrak{l}$ with domain $[\xi_0, \pi]$ that defines the 
lower boundary of the set $R_{h_{\bar{r},\sigma}}$, and which is characterized 
entirely by $\mathcal{L}_{\bar{r},\sigma,\ell_r}(\xi, \mathfrak{l}(\xi), \cdot) = 0$.

Since $\mathfrak{l}$ corresponds exactly to such a constant-level contour, we 
can use derivatives of $\mathcal{L}_{\bar{r},\sigma,\ell_r}$ to compute the derivative of 
$\mathfrak{l}$ with respect to $\xi$.
That is if we define
\begin{equation}
	\gamma^\prime(\xi, \beta)
	\triangleq
	-\frac{
		\partial \mathcal{L}_{\bar{r}, \sigma, \ell_r} / \partial \xi
	}{
		\partial \mathcal{L}_{\bar{r}, \sigma, \ell_r} / \partial \beta
	}(\xi, \beta)
\end{equation}
then $\mathfrak{l}^\prime(\xi) = \gamma^\prime(\xi, \mathfrak{l}(\xi))$.

By extension then, it is possible to derive the \emph{second} derivative of 
$\mathfrak{l}$ using $\mathcal{L}_{\bar{r},\sigma,\ell_r}$ if we define:
\begin{equation}
	\gamma^{\prime\prime}(\xi, \beta) \triangleq
	\tfrac{\partial}{\partial \xi} \gamma^\prime(\xi, \beta)
		+
	\tfrac{\partial}{\partial \beta}\gamma^\prime(\xi, \beta) \cdot \gamma^\prime(\xi, \beta)
\end{equation}
so that $\mathfrak{l}^{\prime\prime}(\xi) = \gamma^{\prime\prime}(\xi, 
\mathfrak{l}(\xi))$. This gives us a sufficient condition for the 
\emph{concavity} of $\mathfrak{l}$.

\begin{proposition}
	Suppose that $\xi_0$ and $\text{graph}(\mathfrak{l}) \triangleq 
	\mathsf{bd}_{[\xi_0,\pi]}$ as above. If for all $(\xi,\beta) \in 
	[\xi_0,\pi] \times [-\beta_\text{max}, \beta_\text{max}]$ we have that
	\begin{equation}\label{eq:property3_min_cond}
		 \gamma^{\prime\prime}(\xi, \beta) < 0
	\end{equation}
	then $\mathfrak{l}$ is concave.
\end{proposition}
\begin{proof}
	Direct from the calculations above.
\end{proof}

\subsection{\texttt{CertifyMin} and the ShieldNN Verifier} 
\label{ssub:certifymin}

Together, \textbf{Properties 1-3} are sufficient to prove that a particular set 
of parameters leads $h_{\bar{r},\sigma}$ to be a barrier function for the KBM. 
Furthermore, each of these conditions involves asserting that 
$\mathcal{L}_{\bar{r},\sigma,\ell_r}$ or its derivatives are strictly positive 
or negative on an interval or a rectangle.

Since $\mathcal{L}_{\bar{r},\sigma,\ell_r}$ is composed of simple functions, it is possible 
for a computer algebra system (CAS) to not only obtain each of these 
verification functions automatically, but to further differentiate each of them 
once \emph{more}. Thus, we can combine an extra derivative with the Mean-Value 
Theorem to verify each of these individual claims. We describe this procedure 
as \texttt{CertifyMin} below.

\begin{algorithm}

\SetKwData{LinFnCnt}{N\_est}
\SetKwData{hyperplanes}{G\_Hinv\_Gtr}
\SetKwData{solnsFound}{AllSolutionsFound}
\SetKwData{numhyperplanes}{NumHyperplanes}
\SetKwData{h}{h}
\SetKwData{b}{b}

\SetKwData{sols}{Solutions}
\SetKwData{cons}{SATConstraints}
\SetKwData{setcons}{setConstraints}
\SetKwData{feas}{Feasible?}
\SetKwData{iis}{IIS}
\SetKwData{z}{z}
\SetKwData{hyperset}{HyperplaneSet}

\SetKwFunction{CntRegions}{EstimateRegionCount}
\SetKwFunction{GetHyperplanes}{GetHyperplanes}
\SetKwFunction{dim}{Dimensions}
\SetKwFunction{sat}{SATsolver}
\SetKwFunction{maxx}{Maximize}
\SetKwFunction{init}{init}
\SetKwFunction{initbools}{createBooleanVariables}
\SetKwFunction{append}{Append}
\SetKwFunction{satq}{SAT?}
\SetKwFunction{checkfeas}{CheckFeas}
\SetKwFunction{truevars}{TrueVars}
\SetKwFunction{selecthypers}{GetHyperplanes}
\SetKwFunction{getiis}{GetIIS}
\SetKwFunction{cntunique}{CountAllUniqueSubsets}

\SetKwFunction{certmin}{CertifyMin}
\SetKwFunction{D}{SymbolicGradient}
\SetKwFunction{norm}{SymbolicNorm}
\SetKwFunction{getcomp}{SymbolicGetComponent}
\SetKwFunction{gditer}{GridIterator}

\SetKwData{df}{Df}
\SetKwData{dfn}{DfNorm}
\SetKwData{scale}{scale}
\SetKwData{gdsz}{gridSize}
\SetKwData{ver}{verified}
\SetKwData{ref}{refine}
\SetKwData{false}{False}
\SetKwData{true}{True}

\SetKw{Break}{break}

\SetKwInOut{Input}{input}
\SetKwInOut{Output}{output}
\Input{function $f(\xi,\beta)$ that is either $\mathcal{L}_{\bar{r},\sigma,\ell_r}$ or one of its derivatives; $\xi$-interval $[\xi_\ell, \xi_h]$; $\beta$-interval $[\beta_\ell, \beta_h]$; a sign $s = \pm1$. \textbf{Either $\xi_\ell=\xi_h$ or $\beta_\ell=\beta_h$ but not both.}}
\Output{\LinFnCnt}
\BlankLine
\SetKwProg{Fn}{function}{}{end}%
\Fn{\certmin{$f$, $\xi_\ell$,$\xi_h$, $\beta_\ell$, $\beta_h$, $s$}}{

	\df $\leftarrow$ \D{f}

	\uIf{$\xi_\ell = \xi_h$}{
		\dfn $\leftarrow$ \norm{\getcomp{\df, $\beta$}}
	}\uElseIf{$\beta_\ell=\beta_h$}{
		\dfn $\leftarrow$ \norm{\getcomp{\df, $\xi$}}
	}\Else{
		\dfn $\leftarrow$ \norm{\df}
	}

	\BlankLine

	\gdsz $\leftarrow$ 1

	\ref $\leftarrow$ \true

	\BlankLine

	\While{\ref}{
		\gdsz $\leftarrow$ \gdsz /10

		\ref $\leftarrow$ \false

		\For{$(\xi^\prime, \beta^\prime)$ {\normalfont\bfseries in} \gditer{\gdsz,$\xi_\ell$,$\xi_h$, $\beta_\ell$, $\beta_h$}}{

			\uIf{$s \cdot f(\xi^\prime, \beta^\prime) < 0$}{
				\Return \false
			}\ElseIf{$\;s \cdot f(\xi^\prime, \beta^\prime) < \sqrt{2}\cdot$ \gdsz $\cdot$ \dfn{$\xi^\prime, \beta^\prime$}$\;$}{
				\ref $\leftarrow$ \true

				\Break
			}
		}

	}

	\BlankLine

	\Return \true
}
\caption{\texttt{CertifyMin}.}
\label{alg:certifymin}
\end{algorithm}

\section{ShieldNN Synthesizer} %
\label{sec:synthesizer}
Given a barrier function, recall from \textbf{(3)} in  \cref{sec:problem} that 
synthesizing a ShieldNN filter entails two components: $\nn_{\negthinspace0}$ 
and $\nn$. That is $\nn_{\negthinspace0}$ chooses a \emph{safe} control for 
each state, and $\nn$ overrides \emph{unsafe} controls with the output of 
$\nn_{\negthinspace0}$.

	\textbf{Design of }$\nn_{\negthinspace0}$\textbf{.} %
	This task is much easier than it otherwise would be, since the ShieldNN 
	verifier also verifies the safe controls as lying between the continuous 
	functions $\max\{-\beta_\text{max}, \mathfrak{l}\}$ and 
	$\min\{\beta_\text{max},\mathfrak{u}\}$ where $\mathfrak{l}$ and is concave 
	and $\mathfrak{u}(\xi) = - \mathfrak{l}(-\xi)$. In particular, then, it is 
	enough to design $\nn_{\negthinspace0}$ as any neural network such that 
	\begin{equation}
		\max\{-\beta_\text{max}, \mathfrak{l}\}
		\leq
		\nn_{\negthinspace0}
		\leq 
		\min\{\beta_\text{max},\mathfrak{u}\}.
	\end{equation}
	This property can be achieved in several ways, including training against 
	samples of $\max\{-\beta_\text{max}, \mathfrak{l}\}$ for example. However, 
	we chose to synthesize $\nn_{\negthinspace0}$ directly in terms of tangent 
	line segments to $\mathfrak{l}$ (and thus exploit the \emph{concavity} of 
	$\mathfrak{l}$). A portion of just such a function $\nn_{\negthinspace0}$ 
	is illustrated by the orange line in \cref{fig:shieldnn_design}. 

	\textbf{Design of }$\nn$\textbf{.} %
	Since the value of $\nn_{\negthinspace0}$ is designed to lie inside the 
	interval of safe controls, the function $\nn_{\negthinspace0}$ can itself 
	be used to decide when an unsafe control is supplied. In particular, using 
	this property and the symmetry $\mathfrak{u}(\xi) = -\mathfrak{l}(-\xi)$, 
	we can simply choose
	\begin{equation}
		\nn : \beta \mapsto \min\{
			\max\{ \nn_{\negthinspace0}(\beta), \beta \},
			-\nn_{\negthinspace0}(-\beta)
			\}.
	\end{equation}
	Note: the closer $\nn_{\negthinspace0}$ approximates its lower bound, 
	$\max\{-\beta_\text{max},\mathfrak{l}\}$, the less intrusive the safety 
	filter. Two constant-$\xi$ slices of such a $\nn$ are shown in 
	\cref{fig:shieldnn_design}.

\section{Extending ShieldNN to Multiple Obstacles}\label{sec:mosa}%
We now consider the general case of multiple obstacles in the environment. We 
propose two approaches: 1) a Single-Obstacle Safe-Action (SOSA). 2) a 
Multi-Obstacle-Safe-Action (MOSA) approach which is described in 
\cref{alg:MOSA}.

\subsubsection{SOSA}
We start by sorting the detected obstacle states during runtime according to 
their distances $r$ from the vehicle. Then, we choose the closest obstacle and 
directly apply the synthesized ShieldNN $\nn$ to the controller steering action 
in order to generate safe steering control actions that are safe considering 
only the current closest obstacle.

\subsubsection{MOSA}
The goal is to search inside the state-action space for an action that is safe 
for all detected obstacles in the environment. Let $\Xi = \{\xi_1, \xi_2, ... 
\xi_N\}$ is the list of angles $\xi$ for all detected obstacles sorted by the 
closest to the furthest obstacle to the vehicle. First, we get the lower and 
upper bounds of the safe action interval for each $\xi_i$ as follows: i) if 
$\xi_i > 0$, then the corresponding safe action interval is 
$[-\beta_\text{max}, \nn_0(\xi)]$ otherwise, the safe action interval is 
$[-\beta_\text{max}, -\nn_0(-\xi)]$. This is due to the symmetry of the 
boundary between safe and unsafe regions as shown in 
\cref{fig:practical_safe_controls}. Second, we find the common safe action 
interval between all detected obstacles by getting the intersection of these 
intervals. If the intersection is an empty interval, this means that we cannot 
find an action that is guaranteed to be safe for all detected obstacles. In 
that case, we fall back to SOSA, generate a safe action considering the closest 
obstacle $\nn(\mathbf{\Xi}[0])$ and raise a flag indicating that no safe action 
can be found. If the unsafe controller is a learning based controller (e.g. a 
neural network), the unsafe action flag is used to train the unsafe controller 
on avoiding moving the vehicle to a state where no common safe action for all 
detected obstacles can be found. When the intersection interval is nonempty, we 
use a safe control action $\beta^*$ inside the intersection that is closest to 
the input unsafe action $\beta$.     
\begin{algorithm}

\SetKwData{LinFnCnt}{N\_est}
\SetKwData{hyperplanes}{G\_Hinv\_Gtr}
\SetKwData{solnsFound}{AllSolutionsFound}
\SetKwData{numhyperplanes}{NumHyperplanes}
\SetKwData{h}{h}
\SetKwData{b}{b}

\SetKwData{sols}{Solutions}
\SetKwData{cons}{SATConstraints}
\SetKwData{setcons}{setConstraints}
\SetKwData{feas}{Feasible?}
\SetKwData{iis}{IIS}
\SetKwData{z}{z}
\SetKwData{hyperset}{HyperplaneSet}

\SetKwFunction{CntRegions}{EstimateRegionCount}
\SetKwFunction{GetHyperplanes}{GetHyperplanes}
\SetKwFunction{dim}{Dimensions}
\SetKwFunction{sat}{SATsolver}
\SetKwFunction{maxx}{Maximize}
\SetKwFunction{init}{init}
\SetKwFunction{initbools}{createBooleanVariables}
\SetKwFunction{append}{Append}
\SetKwFunction{satq}{SAT?}
\SetKwFunction{checkfeas}{CheckFeas}
\SetKwFunction{truevars}{TrueVars}
\SetKwFunction{selecthypers}{GetHyperplanes}
\SetKwFunction{getiis}{GetIIS}
\SetKwFunction{cntunique}{CountAllUniqueSubsets}

\SetKwFunction{mosa}{MOSA}
\SetKwFunction{D}{SymbolicGradient}
\SetKwFunction{norm}{SymbolicNorm}
\SetKwFunction{getcomp}{SymbolicGetComponent}
\SetKwFunction{gditer}{GridIterator}
\SetKwFunction{insert}{insert}

\SetKwData{df}{Df}
\SetKwData{dfn}{DfNorm}
\SetKwData{scale}{scale}
\SetKwData{gdsz}{gridSize}
\SetKwData{ver}{verified}
\SetKwData{ref}{refine}
\SetKwData{false}{False}
\SetKwData{true}{True}
\SetKwData{silb}{SafeIntervalsLBs}
\SetKwData{siub}{SafeIntervalsUBs}
\SetKwData{na}{new Array}
\SetKwData{sa}{SafeAction}
\SetKwData{overlaplb}{IntersectionLB}
\SetKwData{overlapub}{IntersectionUB}
\SetKw{Break}{break}

\SetKwInOut{Input}{input}
\SetKwInOut{Output}{output}
\Input{Array $\mathbf{\Xi}$ contains the list of $\xi$ angles for all detected obstacles sorted by the closest to furthest from the vehicle, $\beta$, $\beta_\text{max}$, $\nn_0, \nn$}
\Output{$\beta^*$, \sa}
\BlankLine
\SetKwProg{Fn}{function}{}{end}%
\Fn{\mosa{$\mathbf{\Xi}$, $\beta$, $\beta_\text{max}$, $\nn_0$, $\nn$}}{

    \silb, \siub $\leftarrow$ \na

    \ForEach {$\xi \in \mathbf{\Xi} $}
    {
        \uIf{$\xi$ > 0}
        {
            \silb.\insert{$\nn_0(\xi)$}
            \BlankLine
            \siub.\insert{$\beta_\text{max}$}
        }
        \Else{
            \silb.\insert{$-\beta_\text{max}$}
            \BlankLine
            \siub.\insert{$\nn_0(\xi)$}
        }
        \overlaplb $\leftarrow \max(\silb)$
        \BlankLine
        \overlapub $\leftarrow \min(\siub)$
    }
    \uIf {$\overlapub \geq \overlaplb$}
    {
        $\sa \leftarrow true$
        
        \uIf {$\beta \in [\overlaplb,  \overlapub]$}
        {
            $\beta^* \leftarrow \beta$
        }
        \uElseIf{$\beta > \overlapub$}
        {
            $\beta^* \leftarrow \overlapub$
        }
        \Else{
            $\beta^* \leftarrow \overlaplb$
        }
    }
    \Else {
        $\beta^* \leftarrow \nn(\mathbf{\Xi}[0], \beta)$, $\sa \leftarrow false$
    }
    \Return $\beta^*$, \sa
}
\caption{\texttt{MOSA}.}
\label{alg:MOSA}
\end{algorithm}

\section{ShieldNN Evaluation}
\label{sec:experiments}

We conduct a series of experiments to evaluate ShieldNN's performance when 
applied to unsafe RL controllers. The CARLA Simulator~\cite{Dosovitskiy17} is 
used as our RL environment, and we consider an RL agent whose goal is to 
drive a simulated vehicle while avoiding the obstacles in the environment. 
The goals of the experiments are to assess the following:
\begin{enumerate}
    \item The safety of the RL agent when ShieldNN is applied after training 
	(Experiment 1).
    \item The robustness of ShieldNN when applied in a different environment 
	than that used in training (Experiment 2).
    \item The effect of applying SOSA and MOSA aproaches on the RL agent in 
	case of having an environment with multiple obstacles in terms of safety 
	(Experiment 3).
\end{enumerate}

		

\textbf{RL Task:} The RL task is to drive a simulated four-wheeled vehicle from 
point A to point B on a curved road that is populated with obstacles. The 
obstacles are static CARLA pedestrian actors randomly spawned at different 
locations between the two points. 
We define  unsafe states as those in which the vehicle hits an obstacle. As 
ShieldNN is designed for obstacle avoidance, we do not consider the states when 
the vehicle hits the sides of the roads to be unsafe with respect to ShieldNN. 
Technical details and graphical representations are included in the 
Supplementary Materials.

\textbf{Reward function and termination criteria:} If the vehicle reaches point 
B, the episode terminates, and the RL agent gets a reward value of a $100$. The 
episode terminates, and the agent gets penalized by a value of a $100$ in the 
following cases: when the vehicle (i) hits an obstacle; (ii) hits one of the 
sides of the road; (iii) has a speed lower than 1 KPH after 5 seconds from the 
beginning of the episode; or (iv) has a speed that exceed the maximum speed (45 
KPH). The reward function is a weighted sum of four terms, and the weights were 
tuned during training. The four terms are designed in order to incentivize the 
agent to keep the vehicle's speed between a minimum speed (35 KPH) and a target 
speed (40 KPH), maintain the desired trajectory, align the vehicle's heading 
with the direction of travel, and keep the vehicle away from obstacles.

\textbf{Integrating ShieldNN with PPO:} We train a Proximal Policy Optimization 
(PPO) \cite{schulman2017proximal} neural network in order to perform the 
desired RL task. To speed up policy learning as in \cite{raffin2019decoupling}, 
we encode the front camera feed into a latent vector using the encoder part of 
a trained $\beta$-Variational Auto-Encoder ($\beta$-VAE). The encoder takes 
160x80 RGB images generated by the simulated vehicle's front facing camera and 
outputs a latent vector that encodes the state of the surroundings. The inputs 
to the PPO Network are: The latent vector $[z_1, ..., z_{dim}]$, the vehicle's 
inertial measurements (current steering angle $\delta_f^c$, speed $v$ and 
acceleration $a$) and the relative angle $\xi$ and distance $r$ between the 
vehicle and the nearest obstacle. The latter two measurements are estimated 
using an obstacle detection module that takes the vehicle's LIDAR data as 
input. In our experiments, we assume we have a perfect obstacle detection 
estimator and we implement it by collecting the ground truth position and 
orientation measurements of the vehicle and the obstacles from CARLA then 
calculating $\xi$ and $r$. The PPO network outputs the new control actions: 
Throttle $\zeta$ and steering angle $\delta_f$. We omit using the brakes as 
part of the control input vector, as it is not necessary for this task. 
However, the RL agent will still be able to slow down the vehicle by setting 
the throttle value to $0$ due to the simulated wheel friction force in CARLA. 
The throttle control action $\zeta$ gets passed directly to CARLA, while the 
steering angle control action gets filtered by ShieldNN. The filter also takes 
$\xi$ and $r$ as input and generates a new safe steering angle 
$\hat{\delta_f}$. To train the VAE, we first collect 10,000 images by driving 
the vehicle manually in CARLA along the desired route with obstacles spawned at 
random locations and observing different scenes from different orientations. We 
train the VAE encoder with cross validation and early-stopping. Then, after 
convergence, we check the output image to validate the VAE encoder.

\textbf{ShieldNN Parameters:} The ShieldNN filter is synthesized as in  
\cref{sec:problem} with $\bar{r} = 4 \text{ m}, \sigma = 0.48$ and KBM 
parameters ${\delta_f}_\text{max} = \pi/4,l_f = l_r = 2m$, and $v_\text{max} = 
20m/s$. 
\subsection{Experiment 1: Safety Evaluation of ShieldNN} 

\begin{table}[!t]
\vspace{1.85mm} %
\centering
\scalebox{0.8}{
\begin{tabular}{|c|c|c|c|c|c|c|c|}
\cline{2-8}
    \multicolumn{1}{l}{} & \multicolumn{2}{|c|}{{\bf Training}} & {\bf Testing} & \multicolumn{2}{|c|}{{\bf Experiment 2}} & \multicolumn{2}{c|}{{\bf Experiment 3A}} \\\hline
    \textbf{Config} & \textbf{Obstacle}   & \textbf{Filter}   & \textbf{Filter}  & \textbf{TC\%$^1$} & \textbf{OHR\%$^2$} & {\bf TC\%$^1$} & {\bf OHR\%$^2$} \\\hline 
    1 & OFF & OFF & OFF & 7.59 & 99.5 & 27.53 & 79.5 \\\hline
    2 & OFF & OFF & ON & 98.82 & 0.5 & 98.73 & 0.5 \\\hline
    3 & ON & OFF & OFF & 94.82 & 8.5 & 71.88 & 34 \\\hline
    4 & ON & OFF & ON & 100 & 0 & 100 & 0 \\\hline
    5 & ON & ON & OFF & 62.43 & 44 & 50.03 & 60 \\\hline
    6 & ON & ON & ON & 100 &  0 & 100 & 0 \\\hline
\end{tabular}
}
\\
 $^1$ TC\% := Track Completion \% \quad $^2$ OHR\% := Obstacle Hit Rate \%


\caption{Experiment 1 \& 2,  
evaluation of safety and performance with and without ShieldNN.}
\label{tab:exp2}
\vspace{-5pt}
\end{table}

The goal of this experiment is to validate the safety guarantees provided by 
ShieldNN when applied to non-safe controllers. To do this, we evaluate the 
three trained agents 
in the same environment they were trained in, and with obstacles 
spawned randomly according to the same distribution used during training. With 
this setup, we consider two evaluation scenarios: (i) when the ShieldNN filter 
is in place 
(ShieldNN ON) and (ii) when ShieldNN filter is not in place 
(ShieldNN OFF). \cref{tab:exp2} shows all six configurations of this 
experiment. For each configuration, we run 200 episodes and record three 
metrics: (i) the minimum distance between the center of the vehicle and the 
obstacles, (ii) the average percentage of track completion, and (iii) the 
percentage of hitting obstacles across the 200 episodes. 

\cref{fig:2-off} and \ref{fig:2-on} show the histograms of the minimum distance 
to obstacles for each configuration. The figure also show two vertical lines at 
2.3 m and 4 m: the former is the minimum distance at which a collision can 
occur, given the length of the vehicle, and the latter is
the value of the safe distance $\bar{r}$ used to design the ShieldNN filter. 
Whenever the ShieldNN was not used in the 200 testing episodes (ShieldNN OFF, 
\cref{fig:2-off}), the average of all the histograms is close to the 2.3 m line 
indicating numerous obstacle collisions. The exact percentage of obstacle hit 
rate is reported in Table \cref{tab:exp2}. Upon comparing the histograms in 
\cref{fig:2-off} with those in \ref{fig:2-on}, we conclude that ShieldNN 
nevertheless renders all the three agents safe: note that the center of mass of 
the histograms shifts above the safety radius parameter, $\bar{r}$, used to 
design the ShieldNN filter. In particular, Agents 2 and 3 were able to avoid 
all the obstacles spawned in all 200 episodes, while Agent 1 hit only 0.5\% of 
the obstacles spawned. Again, we believe this is due to the difference between 
the KBM used to design the filter and the actual dynamics of the vehicle. In 
general, the obstacle hitting rate is reduced by $99.4\%, 100\%$ and $100\%$ 
for Agents 1, 2, and 3, respectively.

\begin{figure}
	\vspace{1.85mm}
	\centering
	\begin{subfigure}{0.8\linewidth}
		\includegraphics[width=\linewidth,trim={60px 0cm 60px 30px},clip]{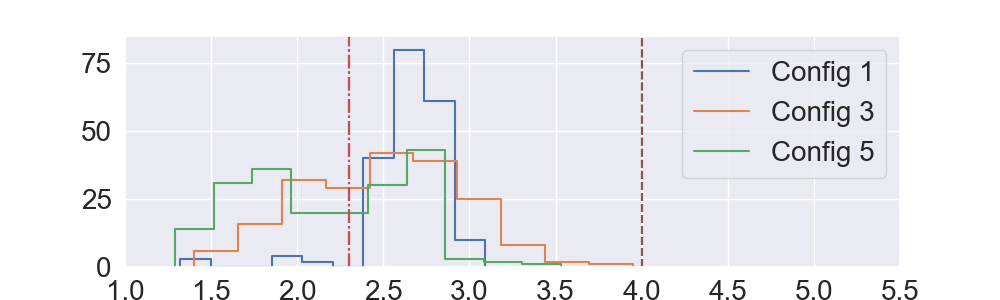}
	\caption{Experiment {\bf 1}, ShieldNN OFF}
	\label{fig:2-off}
	\end{subfigure}
	\\
	\begin{subfigure}{0.8\linewidth}
		\includegraphics[width=\linewidth,trim={60px 0cm 60px 30px},clip]{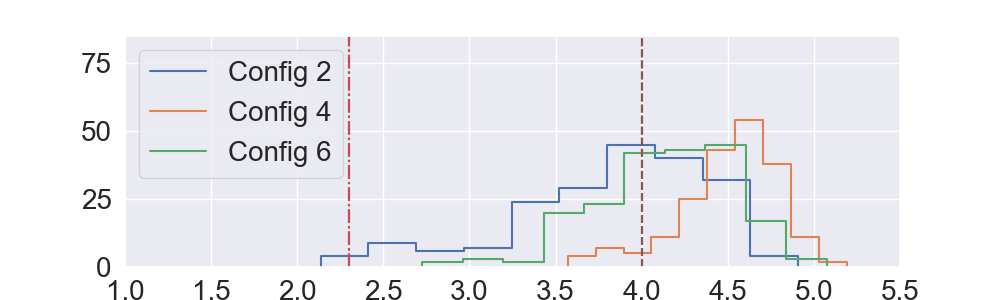}
		\caption{Experiment {\bf 1}, ShieldNN ON}
		\label{fig:2-on}
	\end{subfigure}
	\\
	\begin{subfigure}{0.8\linewidth}
		\includegraphics[width=\linewidth,trim={60px 0cm 60px 30px},clip]{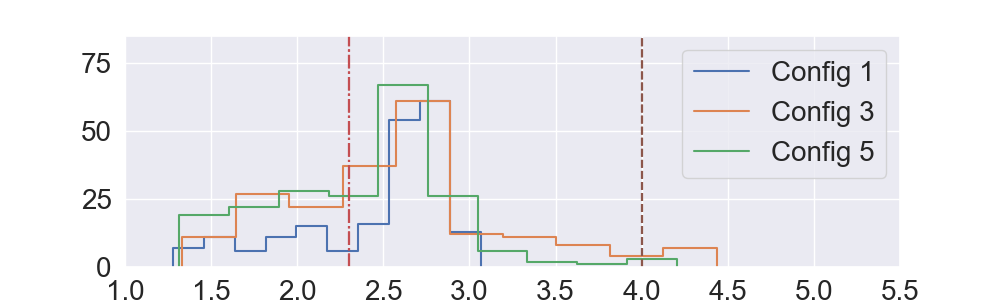}
	\caption{Experiment {\bf 2A}, ShieldNN OFF}
	\label{fig:3a-off}
	\end{subfigure}
	\\
	\begin{subfigure}{0.8\linewidth}
		\includegraphics[width=\linewidth,trim={60px 0cm 60px 30px},clip]{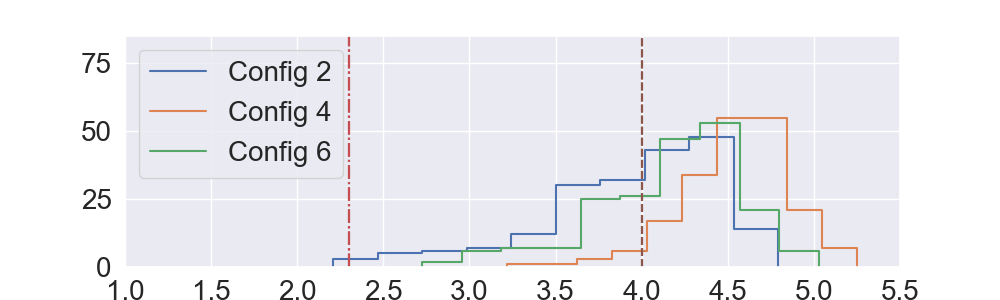}
		\caption{Experiment {\bf 2A}, ShieldNN ON}
		\label{fig:3a-on}
	\end{subfigure}
    %
	\caption{Distributions of distance-to-obstacles for experiments 2 \& 3, with and without ShieldNN.}
	\label{fig:exp3b}
\end{figure}

\subsection{Experiment 2: Robustness of ShieldNN in Different Environments}
The goal of this experiment is to test the robustness of ShieldNN inside a 
different environment than the training environment. We split the experiment 
into two parts:

\underline{\emph{Part 2-A:}} We use the same setup and metrics as in Experiment 
1 , but we perturb the locations of the spawned obstacles by a Gaussian 
distribution $\mathcal{N}(0,\,1.5) \text{m}$ in the lateral and longitudinal 
directions. \cref{fig:3a-off} and \ref{fig:3a-on} show that despite this 
obstacle perturbation, ShieldNN is still able to maintain a safe distance 
between the vehicle and the obstacles whereas this is not the case when 
ShieldNN is OFF. \cref{tab:exp2} shows an overall increase of obstacle hit rate 
and a decrease in track completion rate when ShieldNN is OFF compared to the 
previous experiment. This is expected, as the PPO algorithm is trained with the 
obstacles spawned at locations with a different distribution than the one used 
in testing. However, ShieldNN continues to demonstrate its performance and 
safety guarantees by having almost $100\%$ track completion rate and almost 
$0\%$ obstacle hit rate.

\underline{\emph{Part 2-B:}} 
This experiment is an evaluation of the ability (or not) of RL 
agents equipped with ShieldNN to generalize to novel environments. To evaluate 
performance in this setting, a transfer learning task is implemented where the 
pretrained Agents 2 and 3 are then {\em retrained} for 500 episodes in the new 
environment (compare to 6000 training episodes for the original experiments). 
The new environment is a city road surrounded by buildings, as opposed to the  
highway environment used for original training. This change substantially 
shifts the distribution of the camera input.

\cref{fig:exp3b2} shows the results for Configurations 4 and 6; recall from 
\cref{tab:exp2} in the main text that these configurations represent when 
(re)training is conducted with ShieldNN OFF and ON, respectively. Note that 
configuration 2 -- where the original training was done in an environment with 
no obstacles -- could not successfully complete the track during re-training 
for 500 episodes nor in testing, and is thus not included in the figure. 
Observe that in both configurations, the agent is still able to avoid obstacles 
for the 200 number of test episodes. Furthermore, configuration 6 (both 
retraining/testing with ShieldNN ON), the agent appears to behave more 
conservatively with respect to obstacle avoidance. 
ShieldNN still achieves the desired safety distance on average, and has 
exactly zero obstacle hits in both cases; it also has track completions of 
$98\%$ and $97\%$ respectively.  


\begin{figure}[ht]
    \centering
		\includegraphics[width=0.4\textwidth,trim={60px 0cm 60px 30px},clip]{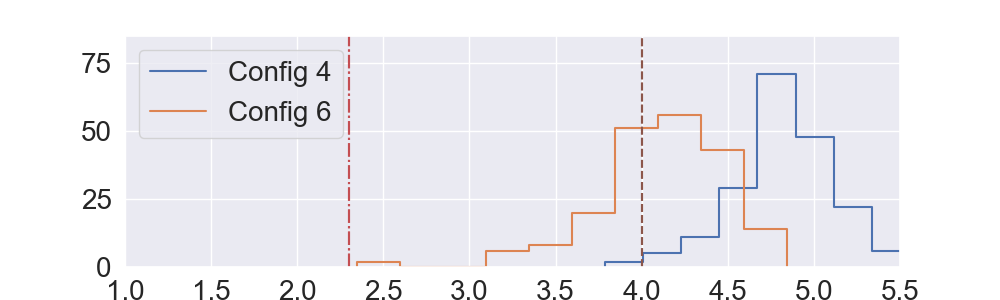}
		\label{fig:hist_3_3}
	\caption{Results of Experiment 3-B, distributions of metrics in a novel environment relative to training.}
	\label{fig:exp3b2}
\end{figure}

\subsection{Experiment 3: Multiple Obstacles}
In order to study the effectiveness of ShieldNN in providing safety in an 
environment where the vehicle can encounter multiple obstacles at the same 
time, we run 100 test scenarios in which we spawn multiple obstacles at random 
positions that are close to each other along the path of the vehicle. We 
compare between four variants based on a pretrained RL agent from Experiment 1 
(Agent 2): 1) No filter, 2) SOSA filter; filtering PPO steering controls based 
on the closest obstacle to the vehicle. 3) MOSA; filtering PPO steering 
controls based on all detected obstacles. 4) MOSA plus penalizing the RL reward 
function with the unsafe action flag. This should disincentivize the RL agent 
from generating control actions that make the vehicle reach a state where MOSA 
cannot output a steering control action that is safe for all obstacles. We 
retrain the RL agent for $2000$ episodes on scenarios that are different from 
the test scenarios. Then, we run the test scenarios with MOSA ON after 
retraining. As shown in \cref{table:exp4}, using the naive approach (SOSA) 
reduces the obstacle hit rate significantly by $75\%$ compared to the case when 
no filter is applied. Applying the proposed MOSA algorithm further reduces the 
obstacle hit rate by $36\%$. The trained RL agent with MOSA applied had the 
lowest obstacle hit rate of $4\%$. This shows that the RL agent is able to 
learn to avoid driving the vehicle into states where no safe action exists for 
all obstacles. Training the RL agent for more episodes does not improve 
obstacle hit rate results further. We notice that track completion percentage 
is relatively low for all test runs which is expected due to the fact that 
trying to avoid hitting multiple obstacles spawned randomly in the environment 
will cause the vehicle to move out of bounds of the track since ShieldNN only 
filters the steering control.

\vspace{-5pt}

\begin{table}[]
\vspace{1.85mm} %
\centering
\begin{tabular}{l|l|l|}

\cline{2-3}
                                                        & \textbf{OHR\%} & \textbf{TC\%} \\ \hline
\multicolumn{1}{|l|}{\textbf{No Filter}}                & 90            & 66            \\ \hline
\multicolumn{1}{|l|}{\textbf{SOSA}}                     & 22             & 73            \\ \hline
\multicolumn{1}{|l|}{\textbf{MOSA} (Before Retraining)} & 14             & 65            \\ \hline
\multicolumn{1}{|l|}{\textbf{MOSA} (After Retraining)}  & 4              & 64            \\ \hline
\end{tabular}
\caption{Experiment 3: Multiple Obstacles Scenarios. SOSA, single-obstacle safe 
action; MOSA, multiple-obstacle safe action; OHR, obstacle hit rate; TC, track 
completion}
\label{table:exp4}
\end{table}


 %


\bibliographystyle{IEEEtran}
\bibliography{mybib} %




%

\clearpage


\appendices  %


\section{Discussion}

\textbf{Question 1: Why is ShieldNN is different from other CBFs?}
The main distinction between ShieldNN and other CBF approaches can be summarized succinctly: ShieldNN does not design a controller at all, much less a \emph{single} safe controller. Rather, ShieldNN designs one NN component that can be used to enforce safety \emph{on any} controller, no matter how that controller was designed. This architecture gives ShieldNN two unique capabilities: (i) ShieldNN can be immediately applied to a black-box controller; and (ii) ShieldNN can be employed \emph{during} RL training of a controller, a situation where the (learned) controller is constantly changing. In particular, optimization-based approaches --- that encode the CBF as a constraint in the numerical optimization problem --- cannot directly be used for the setup in the experimental section where the controller processes data collected from cameras and LiDARs to train a neural network.

\noindent \textbf{Question 2: Since ShieldNN is based on the simple KBM model, will ShieldNN still work for real-world scenarios?}
We address this concern in two ways.

1) Several empirical studies have evaluated the effectiveness of designing controllers for actual four-wheeled vehicles, using only the KBM as a model for their dynamics~\cite{KongKinematicDynamicVehicle2015}. By means of experimental data collected from vehicles on a real-world test track, this reference showed evidence that the KBM is actually a viable approximation to real-vehicle dynamics with regard to controller design.

2) We validated our approach using CARLA, a simulator that simulates a much higher-fidelity dynamical model than the KBM. Even with direct parameter matching between the KBM and the CARLA vehicle, our simulation results show that ShieldNN was almost 100\% effective at avoiding collisions.

Our results do not provide the same evidence of correctness for more complicated, non-KBM dynamical models. However, we view this as a problem for future work: having established a functioning ShieldNN filter for the KBM, we can in the future focus on bounding the errors between the KBM and more complicated dynamical models, thereby obtaining more general evidence of correctness (and with some hope of success, based on points 1 and 2 above).

\noindent \textbf{Question 3: Are there any side effects of ShieldNN?} 
In our experiments, applying ShieldNN during training had the side effect of 
creating a higher curb hitting rate during both training and testing, as 
compared to the case when the agent was trained with ShieldNN OFF. In 
particular, after training for 6000 episodes, the curb hitting rate for agent 2 
went from $100\%$ down to $8\%$. However for agent 3 it went from $100\%$ down 
to $30\%$. This is due to the fact that ShieldNN forces the vehicle to steer 
away from facing an obstacle which, in turn, increases the probability of 
hitting one of the sides of the road. This side effect suggests the possibility 
for future research in generalizing ShieldNN to provide safety guarantees 
against hitting environment boundaries as well. 

\noindent \textbf{Question 4: Are there other current limitations of ShieldNN?}
In addition to the use of the KBM model, another potential limitation of ShieldNN is that it currently applies to single obstacles in a one-at-a-time fashion. We regard the extension to multiple obstacles to be a problem for future work and argue this is a problem  for which the simple safe-control sets identified in ShieldNN will nevertheless be useful. We aim to combine multiple ShieldNNs together in a compositional way to reason about multiple obstacles simultaneously. 




\section{Proofs for Section 4} 
\label{sec:proofs_for_section_4}

There are two claims from \cref{sec:zbf_for_kbm_dynamics} that require proof.

\begin{enumerate}
	\item First, we stated \cref{thm:safe_control_region_grows} without 
		proof.

	\item Second, we claimed that for any KBM parameters $\ell_r = \ell_f$ 
		and $\delta_{f_\text{max}}$, there exists a safety radius, $\bar{r}$, 
		and a barrier parameter, $\sigma$, such that $h_{\bar{r},\sigma}$ and 
		$\alpha_{v_\text{max}}$ (as defined in \eqref{eq:bicycle_barrier} and 
		\eqref{eq:extended_class_k}) comprise a barrier function for the KBM. 
\end{enumerate}
We provide proofs for each of these in the next two subsections after introducing some additional needed notation.

\subsection{Additional Notation} 
\label{sec:additional_notation}

Throughout the rest of this appendix we will use the following notation:
\begin{multline}\label{eq:zbf_expanded}
	\mathscr{L}_{\bar{r},\sigma,\ell_r}(\chi,\beta) \triangleq 
	\nabla_\chi^\text{T} h_{\bar{r},\sigma}(\chi) \cdot f_\text{KBM}(\chi, (\beta, a)) \\
		= v \Bigg(
			\frac{\sigma}{2 \cdot \bar{r} \cdot r} \sin(\xi/2) \sin(\xi - \beta)
			+ \\
			\frac{\sigma}{2 \cdot \bar{r} \cdot \ell_r} \sin(\xi/2) \sin(\beta)
			+
			\frac{\cos(\xi - \beta)}{r^2}
		\Bigg).
\end{multline}
Where $f_\text{KBM}$ is the right-hand side of the ODE in 
\eqref{eq:kbm_dynamics} and the variable $a$ is merely a placeholder, since the 
\eqref{eq:zbf_expanded} doesn't depend on it at all. In particular, 
\eqref{eq:zbf_expanded} has the following relationship with \eqref{eq:zbf_full}:
\begin{equation}
	\mathcal{L}_{\bar{r},\sigma, \ell_r}( \xi, \beta, v)
	=
	\mathscr{L}_{\bar{r},\sigma,\ell_r}((r_\text{min}(\xi),\xi,v), \beta).
\end{equation}

Moreover, we define the following set:
\begin{equation}
	\tilde{\mathcal{C}}_{h_{\bar{r},\sigma}} \triangleq
	\big\{
		\chi^\prime = (r^\prime,\xi^\prime,v^\prime) 
		\; \big\lvert \;
		h(\chi^\prime) \geq 0 \;\wedge\; 0 < v^\prime \leq v_\text{max} 
	\big\},
\end{equation}
which is the subset of the zero-level set of $h_{\bar{r},\sigma}$ that is 
compatible with our assumption that $0 < v \leq v_\text{max}$ (see 
\cref{prob:main_problem}).

\subsection{Proof of \cref{thm:safe_control_region_grows}} 
\label{sub:proof_of_thm:safe_control_region_grows}

We prove the first claim of \cref{thm:safe_control_region_grows} as the 
following Lemma.

\begin{lemma}
\label{lem:safe_region_grows}
Consider any fixed parameters $\bar{r}$, $\ell_r$, $\sigma$ and $v_\text{max} > 
0$. Furthermore, define
\begin{equation}
	K_{\bar{r},\sigma} \triangleq \max(\{1,1/\bar{r}\}) \cdot \left( \tfrac{\sigma}{2 \cdot \bar{r} } 
		+ 2 \right).
\end{equation}
Now suppose that $h_{\bar{r},\sigma}$ is as in \eqref{eq:bicycle_barrier}, and  
$\alpha_{v_\text{max}}$ is as in \eqref{eq:extended_class_k} with $K$ is chosen 
such that $K \geq K_{\bar{r},\sigma}$.

Then for each $(\xi, v, \beta) \in [-\pi,\pi] \times (0, v_\text{max}] \times 
[-\beta_\text{max},\beta_\text{max}]$, the function
\begin{multline}
\label{eq:intermediary_lie_plus_alpha}
	L_{\xi, v, \beta} : r \in [\bar{r}, \infty) \mapsto
		\mathscr{L}_{\bar{r},\sigma,\ell_r}((r,\xi,v),\beta)
		+ \\
		\alpha_{v_\text{max}}( h_{\bar{r},\sigma}((r,\xi,v)) )
\end{multline}
is increasing on its domain, $\text{\normalfont{dom}}(L_{\xi,v,\beta}) = 
[\bar{r},+\infty)$.
\end{lemma}

\begin{remark}
	Note the relationship between the function $L_{\xi,v,\beta}$ in 
	\eqref{eq:intermediary_lie_plus_alpha} and the function used to to define 
	$R_{h,\alpha}$ in \cref{cor:feedback_control_action_set}. That is the set 
	that we are interested in characterizing in 
	\cref{thm:safe_control_region_grows}.
\end{remark}

\begin{proof}
	We will show that when $K \geq K_{\bar{r},\sigma}$, each such function 
	$L_{\xi,v,\beta}$ has a strictly positive derivative on its domain. In 
	particular, differentiating $L_{\xi,v,\beta}$ gives:
	\begin{align}
		\tfrac{\partial}{\partial r} L_{\xi,v,\beta}(r) \negthickspace
		&= \negthickspace
		\tfrac{\partial}{\partial r}\Big[
		\mathscr{L}_{\bar{r},\sigma,\ell_r}\negthinspace((r,\xi,v),\beta) 
			+
			\alpha_{v_\text{max}}\negthinspace( h_{\bar{r},\sigma}((r,\xi,v)) )
		\Big] \notag \\
		&= v \left(
			-\tfrac{\sigma}{2 \cdot \bar{r} \cdot r^2} \sin(\xi/2) \sin(\xi - \beta)
			-2 \tfrac{\cos(\xi - \beta)}{r^3} \right)
			\notag \\ 
			&\qquad\qquad\qquad\qquad\qquad\qquad\qquad\qquad + \tfrac{K \cdot v_\text{max}}{r^2} \notag \\
		&\geq 
		v \left(
			-\tfrac{\sigma}{2 \cdot \bar{r} \cdot r^2} 
			- \tfrac{2}{r^3} \right)
			+ \tfrac{K \cdot v_\text{max}}{r^2}.
	\end{align}
	To ensure that this derivative is strictly positive, it suffices to choose 
	$K$ such that
	\begin{equation}\label{eq:relaxed_positivity_condition}
	 	v \left(
			-\frac{\sigma}{2 \cdot \bar{r} \cdot r^2} 
			-\frac{2}{r^3} \right)
			+ \frac{K \cdot v_\text{max}}{r^2} \geq 0.
	\end{equation}
	For this, we consider two cases: $\bar{r} < 1$ and $\bar{r} \geq 1$.

	When $\bar{r} \geq 1$, then $1/r^3 \leq 1/r^2$ for all $r \geq \bar{r}$. 
	Thus it suffices to choose $K$ such that
	\begin{equation}
			K \geq \frac{v}{v_\text{max}} \left(
			\frac{\sigma}{2 \cdot \bar{r} } +
			2 \right),
	\end{equation}
	which is assured under the assumption that $v \in (, v_\text{max}]$ if 
	\begin{equation}\label{eq:monotone_condition_1}
			K \geq 
			\frac{\sigma}{2 \cdot \bar{r} } +
			2 .
	\end{equation}

	Now, when $\bar{r} < 1$, choosing $K$ according to  
	\eqref{eq:monotone_condition_1} ensures that 
	\eqref{eq:relaxed_positivity_condition} is true for all $r \geq 1$. Thus, 
	we also have to ensure \eqref{eq:relaxed_positivity_condition} holds for 
	$\bar{r} \leq r < 1$. But in this case, $1/r^3 \geq 1/r^2$, so 
	\eqref{eq:relaxed_positivity_condition} will be satisfied if
	\begin{equation}
		K \geq \frac{1}{\bar{r}} \left(
			\frac{\sigma}{2 \cdot \bar{r} } + 2 
		\right).
	\end{equation}
	Thus, the desired conclusion holds if we choose $K \geq K_{\bar{r},\sigma}$ 
	as defined in the statement of the lemma.
\end{proof}

Now, we have the prerequisites to prove \cref{thm:safe_control_region_grows}.

\begin{proof}
	(\cref{thm:safe_control_region_grows}) The first claim of 
	\cref{thm:safe_control_region_grows} is proved as 
	\cref{lem:safe_region_grows}. Thus, it remains to show that for any $\chi = 
	(r, \xi,v) \in \mathcal{C}_{h_{\bar{r},\sigma}}$ with $v \in (0, 
	v_\text{max}]$ --- that is $\chi \in 
	\tilde{\mathcal{C}}_{h_{\bar{r},\sigma}}$ --- we have that 
	\eqref{eq:simple_safe_controls} holds. However, this follows from 
	\cref{lem:safe_region_grows}.

	In particular, choose an arbitrary $\chi^\prime = 
	(r^\prime,\xi^\prime,v^\prime) \in 
	\tilde{\mathcal{C}}_{h_{\bar{r},\sigma}}$, and choose an arbitrary 
	$\omega^\prime = (\beta^\prime, a^\prime) \in 
	R_{h_{\bar{r},\sigma}}((r_\text{min}(\xi^\prime),\xi^\prime,v^\prime))$; as 
	usual we will only need to concern ourselves with the steering control, 
	$\beta^\prime$. First, observe that by definition:
	\begin{multline}
		(\beta^\prime, a^\prime) \in 
		R_{h_{\bar{r},\sigma}}((r_\text{min}(\xi^\prime),\xi^\prime,v^\prime)) \\
		\implies
		\mathscr{L}_{\bar{r},\sigma,\ell_r}((r_\text{min}(\xi^\prime),\xi^\prime,v^\prime),\beta^\prime)
			+ \\
			\alpha_{v_\text{max}}( h_{\bar{r},\sigma}((r_\text{min}(\xi^\prime),\xi^\prime,v^\prime)) )
		\geq 0.
	\end{multline}
	However, the conclusion of this implication can be rewritten using the 
	definition \eqref{eq:intermediary_lie_plus_alpha}: 
	\begin{equation}\label{eq:lem1_notation}
		(\beta^\prime, a^\prime) \in 
		R_{h_{\bar{r},\sigma}}((r_\text{min}(\xi^\prime),\xi^\prime,v^\prime))
		\implies
		L_{\xi^\prime,v^\prime,\beta^\prime}(r_\text{min}(\xi^\prime)) \geq 0.
	\end{equation}
	We now invoke \cref{lem:safe_region_grows}: since $r_\text{min}(\xi^\prime) 
	\geq \bar{r}$ by construction, \cref{lem:safe_region_grows} indicates that 
	$L_{\xi^\prime,v^\prime,\beta^\prime}$ is strictly increasing on the 
	interval $[r_\text{min}(\xi^\prime), r^\prime]$. Combining this conclusion 
	with \eqref{eq:lem1_notation}, we see that 
	$L_{\xi^\prime,v^\prime,\beta^\prime}(r^\prime) \geq 0$. Again using the 
	definition of $L_{\xi^\prime, v^\prime, \beta^\prime}$ in 
	\eqref{eq:lem1_notation}, we conclude that
	\begin{equation}
		\mathscr{L}_{\bar{r},\sigma,\ell_r}((r^\prime,\xi^\prime,v^\prime),\beta^\prime)
			+ \alpha_{v_\text{max}}( h_{\bar{r},\sigma}(r^\prime,\xi^\prime,v^\prime)) )
		\geq 0.
	\end{equation}
	Thus, we conclude that $(\beta^\prime,a^\prime) \in 
	R_{h_{\bar{r},\sigma}}(\chi^\prime)$ by the definition thereof (see the 
	statement of \cref{thm:safe_control_region_grows}). Finally, since 
	$\chi^\prime$ and $\omega^\prime$ were chosen arbitrarily, we get the 
	desired conclusion.
\end{proof}


\subsection{Proof of That a Barrier Function Exists for Each KBM Instance} 
For $h_{\bar{r},\sigma}$ and $\alpha_{v_\text{max}}$ to be a useful class of 
barrier functions, it should be that case that at least one of these candidates 
is in fact a barrier function for each instance of the KBM. We make this claim 
in the form of the following Theorem.

\begin{theorem}\label{thm:parameter_existence}
	Consider any KBM robot with length parameters $\ell_r = \ell_f$; maximum 
	steering angle $\delta_{f_\text{max}}$; and maximum velocity $v_\text{max} 
	> 0$. Furthermore, suppose that the following two conditions hold:
	\begin{enumerate}
		\item $\beta_\text{max} \leq \pi/2$, or equivalently, 
			${\delta_f}_\text{max} \leq \tfrac{\pi}{2}$;

		\item $\tfrac{1}{\ell_r}\left( \sigma(1-\sigma)\ell_r + \sigma 
			\bar{r} \right) \sin(\tfrac{\pi}{4} + 
			\tfrac{\beta_\text{max}}{2})\sin(\beta_\text{max}) \geq 2$; and
 
	\end{enumerate}

	Then for every $\chi = (r, \xi, v)$ such that $0 < v \leq v_\text{max}$ the 
	set $R_{h_{\bar{r},\sigma}}(\chi)$ is non-empty. In particular, the 
	feedback controller (interpreted as a function of $\xi$ only):
	\begin{equation}
		\pi : \xi \mapsto
		\begin{cases}
			-\beta_\text{max} & \xi < -\epsilon \\
			\xi & \xi \in [-\beta_\text{max}, \beta_\text{max}] \\
			\beta_\text{max} & \xi > \epsilon
		\end{cases}
	\end{equation}
	is safe.
\end{theorem}

\begin{remark}
Note that there is always a choice of $\bar{r}$ and $\sigma \in (0,1)$ such 
that condition {\itshape (ii)} can be satisfied. In particular, it suffices for 
$\bar{r}$ and $\sigma$ to be chosen such that:
\begin{equation}\label{eq:translated_r_sigma_condition}
	2\frac{(\ell_r/\bar{r})}{\sin(\beta_\text{max}) \sin(\pi/4 + 
	\beta_\text{max}/2)} \leq \sigma.
\end{equation}
Thus, by making $\bar{r}$ large enough relative $\ell_r$, it is possible to 
choose a $\sigma \in (0,1)$ such that the inequality 
\eqref{eq:translated_r_sigma_condition} holds, and {\itshape (ii)} is 
satisfied.  
\end{remark}
\begin{proof}
	{\itshape (\cref{thm:parameter_existence})} As a consequence of 
	\cref{thm:safe_control_region_grows}, it is enough to show that 
	$R_{h_{\bar{r},\sigma}}((r_\text{min}(\xi),\xi,v_\text{max}))$ is non-empty 
	for every $\xi \in [-\pi,\pi]$.

	The strategy of the proof will be to consider the control $\beta = 
	\pi(\xi)$, and verify that for each $\chi = (r, \xi, v) \in 
	\tilde{\mathcal{C}}_{h_{\bar{r},\sigma}}$ such that $\xi \in [0, \pi]$, we 
	have:
	\begin{equation}\label{eq:sub_condition}
		\mathcal{L}_{\bar{r},\sigma, \ell_r}( \xi, \pi(\xi)) \geq 0.
	\end{equation}
	The symmetry of the problem will allow us to make a similar conclusion for 
	$\xi \in [-\pi, 0]$.

%
We proceed by partitioning the interval $[0,\pi]$ into the following three 
intervals:
	\begin{equation*}
	 	I_1 \triangleq [0, \beta_\text{max}], \;\;
	 	I_2 \triangleq (\beta_\text{max}, \pi/2 + \beta_\text{max}], \;\;
	 	I_3 \triangleq (\pi/2 + \beta_\text{max}, \pi].
	 \end{equation*}
and consider the cases that $\xi$ is in each such interval separately. 

\noindent\textbf{Case 1 $\xi \in I_1$:} In this case, $\pi(\xi) = \xi$, and 
$\xi \leq \beta_\text{max} \leq \pi/2$ by assumption. It is direct to show that:
	\begin{equation}
		\cos(\xi - \pi(\xi)) = \cos(0) \geq 0
	\end{equation}
	and
	\begin{equation}\label{eq:case1_sin_lower_bound}
	\sin(\xi/2)\sin(\xi - \pi(\xi)) = 0.
	\end{equation}
	Hence, the $\cos$ term in \eqref{eq:zbf_full} can be lower bounded by zero, 
	and the first term in \eqref{eq:zbf_full} is identically zero by  
	\eqref{eq:case1_sin_lower_bound}. Thus, in this case, \eqref{eq:zbf_full} 
	is lower bounded as as:
	\begin{equation}
		\mathcal{L}_{\bar{r},\sigma, \ell_r}( \xi, \beta_\text{max}) \geq 
		\frac{\sigma \cdot v \cdot \sin(\xi/2) \sin(\pi(\xi))}{2 \cdot \bar{r} \cdot \ell_r},
	\end{equation}
	which of course will be greater than zero since $\xi \in I_1 = 
	[0,\beta_\text{max}]$ with $\beta_\text{max}\leq \pi/2$ by assumption 
	\emph{(i)}.

\noindent\textbf{Case 2 $\xi \in I_2$:} In this case, $\pi(\xi) = 
\beta_\text{max}$. Thus, for $\xi \in I_2$, we have that:
	\begin{align}
		\cos(\xi - \beta_\text{max}) &\geq 0 \\
		\sin(\xi/2)\sin(\xi - \beta_\text{max}) &\geq 0 \\
		\sin(\xi/2)\sin(\beta_\text{max}) &\geq 0.
	\end{align}
	Consequently, \eqref{eq:sub_condition} is automatically satisfied, since 
	all of the quantities in the Lie derivative are positive.

\noindent\textbf{Case 3 $\xi \in I_3$:} In this case, $\pi(\xi) = 
\beta_\text{max}$ as in Case 2. However, the $\cos$ term is now negative in 
this case:
	\begin{equation}
		0 > \cos(\xi - \beta_\text{max}) \geq -\frac{1}{\bar{r}^2}.
	\end{equation}
	Thus, since the other two terms are positive on this interval, we need to 
	have:
	\begin{equation}
		\sin(\tfrac{1}{2}(\tfrac{\pi}{2} + \beta_\text{max})) \negthinspace
		\left(
			\tfrac{\sigma(1-\sigma)}{2 \cdot \bar{r}^2} 
				\sin(\pi - \beta)
			+
			\tfrac{\sigma}{2 \cdot \bar{r} \cdot \ell_r}
				\sin(\beta)
		\right)
		\geq \tfrac{1}{\bar{r}^2}.
	\end{equation}
	This follows because on $I_3$, $\sin(\tfrac{\xi}{2}) \geq 
	\sin(\tfrac{1}{2}(\tfrac{\pi}{2} + \beta_\text{max}))$ and $\sin(\xi - 
	\beta_\text{max}) \geq \sin(\pi - \beta_\text{max}) $; i.e. we substituted 
	the lower and upper end points of $I_3$, respectively. Noting that 
	$\sin(\pi - \beta_\text{max}) = \sin(\beta_\text{max})$, we finally obtain:
	\begin{equation}
		\sin(\tfrac{1}{2}(\tfrac{\pi}{2} + \beta_\text{max})) \sin(\beta_\text{max})
		\left(
			\tfrac{\sigma(1-\sigma)}{2} 
			+
			\tfrac{\sigma \bar{r}}{2 \cdot\ell_r}				
		\right)
		\geq 1.
	\end{equation}
	The preceding is just another form of \emph{(ii)} so we have the desired 
	conclusion in \eqref{eq:sub_condition}.

	The conclusion of the theorem then follows from the combined consideration 
	of Cases 1-3 and \cref{thm:safe_control_region_grows} as claimed above.
\end{proof}




%



%



%





\end{document}